\documentclass[pdflatex,sn-mathphys-ay]{sn-jnl}

\usepackage[utf8]{inputenc} 
\usepackage[T1]{fontenc}    
\usepackage{hyperref}       
\usepackage{url}            
\usepackage{booktabs}       
\usepackage{amsfonts}       
\usepackage{nicefrac}       
\usepackage{microtype}      
\usepackage{xcolor}         

\usepackage{placeins}

\usepackage{amsmath,amsthm,amssymb}
\usepackage{graphicx}
\usepackage{subcaption}
\usepackage[separate-uncertainty=true]{siunitx}

\usepackage{url}

\DeclareMathOperator{\prob}{P}
\DeclareMathOperator{\E}{E}
\DeclareMathOperator{\Var}{Var}

\DeclareMathOperator{\diag}{diag}

\DeclareMathOperator{\Gmom}{GMoM}
\DeclareMathOperator{\Gmed}{GMed}
\DeclareMathOperator{\argmin}{argmin}
\DeclareMathOperator{\tr}{tr}

\DeclareMathOperator{\Vect}{vec}

\DeclareMathOperator{\linearspan}{span}

\usepackage{mleftright} 
\newcommand\pr[1]{\prob\mleft(\,#1\,\mright)}
\newcommand\ex[2][]{\E_{#1}\mleft[\,#2\,\mright]}

\newcommand\gmed[1]{\Gmed\mleft(\,#1\,\mright)}
\newcommand\gmom[2]{\Gmom_{#1}\mleft(\,#2\,\mright)}

\newcommand\cprob[2]{\prob\mleft(\,#1\;\middle|\;#2\,\mright)}

\newcommand\vect[1]{\Vect\left(\,#1\,\right)}
\newcommand\xvect[1]{\mathbf{x}^{(#1)}}
\newcommand\xrandvect{\mathbf{x}}
\newcommand\Gammanaught{\mathbf{\Gamma}_0}
\newcommand\gnaught{\mathbf{g}_0}
\newcommand{\vertiii}[1]{{\left\vert\kern-0.15ex\left\vert\kern-0.15ex\left\vert #1 \right\vert\kern-0.15ex\right\vert\kern-0.15ex\right\vert}}

\newcommand{\bs}[1]{\boldsymbol{#1}}

\newtheorem{theorem}{Theorem}[section]
\newtheorem*{theorem*}{Theorem} 
\newtheorem{corollary}[theorem]{Corollary}
\newtheorem{lemma}[theorem]{Lemma}
\newtheorem{definition}[theorem]{Definition}
\newtheorem{example}[theorem]{Example}

\newgeometry{left=1.3in,right=1.3in,top=1in,bottom=1in}

\begin{document}

\title{Robust Score Matching}

\author*[1]{\fnm{Richard} \sur{Schwank}}\email{richard.schwank@tum.de}

\author[2]{\fnm{Andrew} \sur{McCormack}}\email{mccorma2@ualberta.ca}

\author[1,3]{\fnm{Mathias} \sur{Drton}}\email{mathias.drton@tum.de}

\affil*[1]{\emph{\orgdiv{School of Computation, Information and Technology}}, \orgname{Technical University of Munich},
\country{Germany}}

\affil[2]{\emph{\orgdiv{Department of Mathematical and Statistical Sciences}}, \orgname{University of Alberta}, 
\country{Canada}}

\affil[3]{\emph{\orgdiv{Munich Center for Machine Learning}}, \city{Munich}, \country{Germany}}

\abstract{
  Proposed in \citet{HyvaraninScoreM}, score matching is a parameter estimation procedure that does not require computation of distributional normalizing constants. In this work we utilize the geometric median of means to develop a robust score matching procedure that yields consistent parameter estimates in settings where the observed data has been contaminated. A special appeal of the proposed method is that it retains convexity in exponential family models. The new method is therefore particularly attractive for non-Gaussian, exponential family graphical models where evaluation of normalizing constants is intractable. Support recovery guarantees for such models when contamination is present are provided. Additionally, support recovery is studied in numerical experiments and on a precipitation dataset. We demonstrate that the proposed robust score matching estimator performs comparably to the standard score matching estimator when no contamination is present but greatly outperforms this estimator in a setting with contamination.
}

\maketitle

\section{Introduction}
\label{sec:intro}

Detecting and mitigating the influence of outliers or contaminated observations in multivariate data is a challenging task \citep{MaronnaRobustStatistics},  particularly in high-dimensional settings where there are many possible ways in which an observation can be deemed an outlier and where computational considerations play an important role  \citep{Diakonikolas2023HighDimRobustStats}.
In this work we take up the problem of designing robust estimators of high-dimensional joint densities from exponential family models. The solution we propose is a robustified version of score matching, developed with the help of a carefully chosen multivariate median-of-means technique.

 An exponential family consists of a collection of probability distributions that have densities of the form
 \begin{align}
\label{eqn:expfamdens}
     p(\xrandvect|\bs{\theta}) = \exp\left( \bs{\theta}^\top \mathbf{t}(\xrandvect) - a(\bs{\theta}) + b(\xrandvect)  \right), \quad \xrandvect\in\mathcal{X}.
\end{align}
The parameter $\bs{\theta}$ ranges over the natural parameter space $\Omega$, which is comprised of points where the integral  
$\int_{\mathcal{X}} \exp(\bs{\theta}^\top \mathbf{t}(\xrandvect) + b(\xrandvect)) d\xrandvect = \exp(a(\bs{\theta}))$ is finite. Aside from very special cases like Gaussian models, estimating the parameter $\bs{\theta}$ in \eqref{eqn:expfamdens} by maximum likelihood is not feasible: in general models \citep{SunRegScoreMatchingNeurips, Roy2020CountGraphicalModelIntractableNormalizingConstant} the normalizing constant $\exp(a(\bs{\theta}))$ does not have a closed-form expression and must be found by expensive numerical integration, a problem that is exacerbated by the fact that  maximizing the likelihood typically requires iterative optimization procedures. 

Score matching (SM), proposed by \citet{HyvaraninScoreM}, is an estimation procedure that avoids the aforementioned shortcomings of maximum likelihood estimation in exponential families for continuous data. It does not require that the normalizing constant $\exp(a(\bs{\theta}))$ be known. Moreover, score matching amounts to minimizing a convex, quadratic loss function in $\bs{\theta}$, a task that is easily solved even when $\bs{\theta}$ is high-dimensional.  One prominent application of score matching is estimation of non-Gaussian graphical models \citep{YuScoreMatchNonNeg} that may be formed by structuring the  sufficient statistics $\mathbf{t}(\xrandvect)$ in~\eqref{eqn:expfamdens} so that the vanishing of components of $\bs{\theta}$ correspond to conditional independence relations
between variables \citep{LauritzenGraphicalModels}. 

The central contribution of this work is to extend the score matching methodology to handle data that has been corrupted or contains outliers. Specifically, we propose using the geometric median of means (GMoM) \citep{MinskerGeometricMoM} to robustify the quadratic empirical loss function for exponential family score matching.  Crucially, by using the geometric median of means the robustified objective function remains convex. This property is in general not preserved by other multivariate medians, such as the componentwise median.
The GMoM interpolates between the mean and geometric median of data points, with a block-size parameter determining the relative proximity of the GMoM to the mean and geometric median. By 
altering the block-size parameter we show how the proposed method can be tuned to handle different levels of corruption.

Our paper is structured as follows. Sections~\ref{subsec:scorematching}-\ref{subsec:gmom} review score matching and the geometric median of means, respectively. Section \ref{sec:robustsm} details the proposed robust score matching procedure, provides theoretical results on robustness against contamination, and discusses hyperparameter tuning. Section \ref{sec:highDimGraphModel} introduces an $\ell_1$-regularized version of  robust score matching for graphical models and presents a support recovery guarantee under corruption. Simulations in Section \ref{sec:experiments} illustrate the efficacy of robust score matching. This section concludes by applying the proposed procedure to a data set on precipitation in the Alps. Proofs of all theorems can be found in the appendix.

\textbf{Notation:} Scalars are denoted by $\alpha$, $x$, etc., vectors by $\xrandvect$, $\bs{\theta}$, etc., and matrices by $\mathbf{X}, \mathbf{\Theta}$, etc. Subscripts $X_{ij}$ and $x_i$ indicate matrix and vector components, while the superscripts on $\xvect{i}$ index different observations in a random sample. The gradient is denoted by $\nabla=(\partial_1,\ldots,\partial_p)$ and $\partial_{jj}$ denotes the second partial derivatives with respect to an argument $x_j$. Important vector norms are the Euclidean norm $\|\cdot\|_2$, the Manhattan norm $\|\cdot\|_1$, and the maximum norm $\|\cdot\|_\infty$. For a matrix $\mathbf{A}$, $\vertiii{\mathbf{A}}_{\infty,\infty}=\max_{i=1,\ldots,a}\sum_{j=1}^b |A_{ij}|$, and $\tr(\mathbf{A})$ and $\diag(\mathbf{A})$ are the trace and the diagonal part, respectively. Finally, $1\{b\}$ is the indicator function that equals $1$ if the boolean $b$ is true and zero otherwise.

\section{Preliminaries}
\label{sec:prelims}

\subsection{Generalized score matching}
\label{subsec:scorematching}
In this section, we review the main aspects of score matching in the generalized form introduced by \citet[Sect.~2]{YuScoreMatchNonNeg}.

Suppose that $n$ i.i.d.~observations $\xvect{1},\ldots,\xvect{n} \in \mathbb{R}^m$ are sampled from a distribution in an $r$-dimensional exponential family $\mathcal{P} = \{P_{\bs{\theta}}:\bs{\theta} \in \Omega \subseteq \mathbb{R}^r\}$ with densities of the form \eqref{eqn:expfamdens}. In our use cases, the densities are considered with respect to Lebesgue measure on $\mathcal{X}=\mathbb{R}^m$ or $\mathcal{X}=[0,\infty)^m$. Write $\mathbf{X}$ for the $n\times m$ matrix of all observations.  For practical consideration of general families $\mathcal{P}$, it is crucial to be able to obtain an estimate of the true parameter value $\bs{\theta}_0$ that does not require the computation of the normalizing constant $\exp(-a(\bs{\theta}))$. This may be achieved via the generalized score matching estimator, which minimizes an empirical loss function that approximates the loss function  
\begin{equation}
\label{eqn:scorematloss}
  J_h(\bs{\theta}) =  \frac{1}{2}\int_{\mathcal{X}} \Vert \nabla_\mathbf{x} \log(p(\mathbf{x}|\bs{\theta})) \circ \mathbf{h}(\mathbf{x})^{1/2} -
  \nabla_\mathbf{x} \log(p(\mathbf{x}|\bs{\theta}_0)) \circ \mathbf{h}(\mathbf{x})^{1/2} \Vert_2^2 \; p(\mathbf{x}|\bs{\theta}_0) \mathrm{d}\mathbf{x},
\end{equation}
where gradient $\nabla_\mathbf{x}$ is taken with respect to the data  $\mathbf{x}$, $\circ$ is the componentwise product of vectors, and $\mathbf{h}(\mathbf{x})^{1/2} = (h_1(\mathbf{x})^{1/2},\ldots,h_m(\mathbf{x})^{1/2})$ comprises the square roots of $m$ non-negative weight functions.  Nontrivial weighting is needed when the support $\mathcal{X}$ is constrained.  The loss $J_h(\bs{\theta})$ in \eqref{eqn:scorematloss} is the expected weighted squared distance between the true score function and the score function $\nabla_\mathbf{x} \log(p(\mathbf{x}|\bs{\theta}))$ at $\bs{\theta}$.  Under mild conditions on $\mathbf{h}$, the unique minimizer of \eqref{eqn:scorematloss} is $\bs{\theta}_0$ \citep[Prop 2]{YuScoreMatchNonNeg}.

A key property of $J_h(\bs{\theta})$ is that after an integration by parts, $J_h(\bs{\theta})=\tfrac{1}{2}\bs{\theta}^\top\Gammanaught\bs{\theta} - \gnaught^\top\bs{\theta}$ up to constants not depending on $\bs{\theta}$, with $\Gammanaught=\ex[\bs{\theta}_0]{\mathbf{\Gamma}(\xrandvect)}$ and $\gnaught=\ex[\bs{\theta}_0]{\mathbf{g}(\xrandvect)}$ where 
\begin{align}
\label{eqn:GamFormula}
    \mathbf{\Gamma}(\xrandvect) =& \sum_{j = 1}^m h_j(\xrandvect)\partial_j\mathbf{t}(\xrandvect) \partial_j\mathbf{t}(\xrandvect)^\top \in \mathbb{R}^{r \times r},
    \\
    \label{eqn:gFormula}
    \mathbf{g}(\xrandvect) =& -\sum_{j=1}^m \big(h_j(\xrandvect)\partial_j b(\xrandvect) \partial_j \mathbf{t}(\xrandvect) + \\
    \notag
    & h_j(\xrandvect) \partial_{jj}\mathbf{t}(\xrandvect) + \partial_j h_j(\xrandvect) \partial_j\mathbf{t}(\xrandvect)\big) \in \mathbb{R}^r.
\end{align}

The empirical loss $J_h(\bs{\theta}; \mathbf{X})$ replaces $\Gammanaught$ by $\overline{\mathbf{\Gamma}}(\mathbf{X}):=\frac{1}{n}\sum_{i=1}^n \mathbf{\Gamma}(\xvect{i})$ and $\gnaught$ by $\overline{\mathbf{g}}(\mathbf{X}):=\frac{1}{n}\sum_{i=1}^n \mathbf{g}(\xvect{i})$. The resulting score matching estimator is
\begin{align}
\label{eqn:emp_sm}
    \hat{\bs{\theta}}(\mathbf{X}) := \underset{\bs{\theta}\in\Omega}{\argmin}\; J_h(\bs{\theta}; \mathbf{X}) 
    = \overline{\mathbf{\Gamma}}(\mathbf{X})^{-1}\overline{\mathbf{g}}(\mathbf{X}).
\end{align}
The score matching estimator $\hat{\bs{\theta}}$ depends on the weighting function $\mathbf{h}$ through $\mathbf{\Gamma}$ and $\mathbf{g}$. When the sample space $\mathcal{X}$ is $\mathbb{R}^m$ the constant weighting $\mathbf{h}(x) = (1,\ldots,1)$ can be used. When $\mathcal{X}$ has a boundary, a suitable choice of $\mathbf{h}$ can dampen boundary effects to ensure that the integration by parts argument is valid \citep{HyvarinenExtensions,YuScoreMatchNonNeg}. 

Under mild regularity conditions, $\hat{\bs{\theta}}$ consistently estimates $\bs{\theta}_0$ as $n \rightarrow \infty$. In high-dimensional scenarios where $r > n$, the matrix $\overline{\mathbf{\Gamma}}(\mathbf{X})$ is not invertible and the score matching estimator does not exist. To handle such settings, \citet{YuScoreMatchNonNeg} modify the objective function $J_h(\bs{\theta}; \mathbf{X})$ in \eqref{eqn:emp_sm} by adding positive offsets to the diagonal entries of $\overline{\mathbf{\Gamma}}(\mathbf{X})$.

\begin{example}[Square Root Graphical Model]\label{example:sqr}
Consider non-negative data with $\mathcal{X}=[0,\infty)^m$.  
The square root graphical model is parametrized by a pair $\bs{\theta}=(\mathbf{\Theta}, \bs{\eta})$ with $\mathbf{\Theta}\in\mathbb{R}^{m\times m}$ and $\bs{\eta}\in\mathbb{R}^m$, and has densities
\begin{equation}
    p(\xrandvect|\bs{\theta}) = \exp\bigg(- \sum_{i=1}^m \Theta_{ii} x_i
    + \sum_{1\leq i < j \leq m} 2\Theta_{ij}x_i^{1/2}x_j^{1/2} + 
    \sum_{i=1}^m 2\eta_i x_i^{1/2} - a(\bs{\theta})\bigg)
\end{equation}
with respect to Lebesgue measure \citep{InouyeSquareRootModel}. This is an example of a pairwise interaction model \citep{Mingyu2016statistical} where the interaction parameter $\Theta_{ij}$ represents the degree of dependence
between $x_i$ and $x_j$ conditionally on all of the other components. In particular, if $\Theta_{ij} = 0$ then $x_i$ is conditionally independent of $x_{j}$ given $\{x_k: k \neq i,j\}$.  In the formalism of graphical models, these independencies can be expressed in an undirected graph \citep{handbook}. 
The normalizing constant in this model is intractable, making score matching an attractive approach.
\end{example}

\subsection{Geometric median of means and robustness}
\label{subsec:gmom}

Recent literature has popularized the (univariate) median-of-means (MoM) as a robust mean estimator \citep{DevroyeGMoMNotSubGaus,LaforgueMoMcorrupt}.
The geometric median of means (GMoM) is a multivariate generalization of the MoM \citep{MinskerGeometricMoM}.

\begin{definition}\label{def:gmom}
    The GMoM of $\xvect{1},\ldots, \xvect{n}\in\mathbb{R}^p$ with $K$ (a divisor of $n$) blocks is defined as 
    \begin{equation*}\gmom{K}{\xvect{1},\ldots, \xvect{n}}:=\gmed{\hat{\bs{\mu}}^{(1)}, \ldots, \hat{\bs{\mu}}^{(K)}}, 
    \end{equation*} where $\hat{\bs{\mu}}^{(j)}$ is the sample mean of $\xvect{(j-1)K + 1},\ldots, \xvect{jK}$ and 
    $\Gmed$ denotes the \emph{geometric median} defined as 
    \begin{align}\label{eqn:GmedDef}
    \hspace{-.1cm}\gmed{\hat{\bs{\mu}}^{(1)}, \ldots, \hat{\bs{\mu}}^{(K)}} = \underset{\mathbf{m} \in \mathbb{R}^p}{\argmin} \sum_{i = 1}^K \Vert \hat{\bs{\mu}}^{(i)} - \mathbf{m} \Vert.
\end{align}
\end{definition}
When $p=1$, the GMoM reduces to the MoM, because the geometric median of real numbers equals the standard median. In this case, the MoM partitions observations into $K$ blocks, computes the sample mean within each block, and then aggregates the block means by taking a sample median.  Hence, the MoM is an interpolation between the sample mean ($K=1$) and the sample median ($K=n$). For intermediate values of $K$, the MoM inherits robustness properties of the geometric median while also being an approximately unbiased estimator of the population mean. If the block-sizes of the MoM are increasing, asymptotically the MoM is a consistent estimator of the population mean \citep[Sect.~2.5]{MinskerMoMReview}. An advantage of the MoM over the mean is that if the moment generating function of the population distribution does not exist the sample mean will concentrate around the mean at a polynomial rate, whereas the MoM achieves sub-Gaussian concentration when second moments exist; see \citet{LugosiMeanEstimationHeavyTails} for an in-depth discussion.

If the ambient dimension $p$ is larger than one, the GMoM inherits concentration properties and robustness against outliers from its univariate counterpart, as shown in Section \ref{sec:robustsm}. We consider outliers originating from the contamination of entire observations, also referred to as \emph{rowwise} corruption. See Section \ref{subsec:contamination_assumptions} for details. A basic quantity to assess robustness against rowwise corruption is the \emph{breakdown point} \citep{LopuhaBreakdownPoints}. It is the minimal proportion of observations that, if tampered with arbitrarily, can force the estimator to diverge to infinity. 

In principle, one could generalize the MoM to higher dimensions using any multivariate median concept \citep[see the survey of][]{SmallMedianSurvey} instead of the geometric median. However, subleties arise for our later application in robust score matching as we seek robust estimates of a collection of many means that feature in a loss that ought to admit a well-defined minimizer.  For this reason, a candidate median concept for robust score matching should satisfy the following properties:
 \begin{itemize}
    \item[\textbf{(R1)}] The median should be a convex combination of its arguments. This is to ensure that the median of positive semidefinite matrices is again positive semidefinite, which is needed for applying the GMoM to $\mathbf{\Gamma}(\xrandvect)$ (Section \ref{sec:robustsm}).
    \item[\textbf{(R2)}] Computation should be feasible in high dimensions. This is because the number of parameters in a graphical model scales quadratically with the number of nodes.
    \item[\textbf{(R3)}] The median should have a high breakdown point against rowwise contamination.
\end{itemize}

Many high-dimensional estimation problems can be addressed surprisingly well by seemingly simple coordinate-wise procedures.  However, our argument against a componentwise median is that it fails to satisfy \textbf{(R1)}.  In practice, this entails (robustly) estimated loss functions that end up being unbounded below, with no associated score matching estimator.  In contrast, the geometric median $\mathbf{m}$, if it does not equal one of its arguments $\hat{\bs{\mu}}^{(1)}, \ldots, \hat{\bs{\mu}}^{(K)}$, can be rewritten as \begin{align}\label{eqn:GmedIteration}
\mathbf{m} = 
\frac{1}{\sum_{i=1}^K 1/\|\mathbf{m} - \hat{\bs{\mu}}^{(i)}\|_2}\sum_{i=1}^K \frac{\hat{\bs{\mu}}^{(i)}}{\|\mathbf{m} - \hat{\bs{\mu}}^{(i)}\|_2} \end{align} by setting gradient with respect to $\mathbf{m}$ in \eqref{eqn:GmedDef} to zero. The GMed thus satisfies \textbf{(R1)}.

Regarding the computational requirement \textbf{(R2)}, note that while there is not a closed-form solution for the geometric median, equation \eqref{eqn:GmedIteration} immediately suggests a fixed point algorithm called \emph{Weiszfeld's algorithm}. The computational complexity of a single iteration step is only $\mathcal{O}(p K)$. In contrast, other well-known multivariate medians often have exponential complexity in the ambient dimension $p$; see \citet{RonkainenOjaMedComplexity} for the \emph{Oja median} and \citet{LiuTukeyMedComplexity} for the \emph{Tukey median}. Convergence of Weiszfeld's algorithm is guaranteed under slight modifications that prevent getting stuck on the input vectors \citep{VardiWeiszfeldImproved}.

Lastly, the geometric median satisfies \textbf{(R3)} as its breakdown point is $\lfloor (K+1)/2\rfloor / K$ \citep{LopuhaBreakdownPoints}, the same as that of the univariate median. In contrast, the breakdown point of the Oja median tends to zero with the sample size $n$ \citep{NiinimaaOja0Bd}. It is between $1/3$ and $1/(p+1)$ for the Tukey median \citep{DonohoBreakdownP}.

While our approach uses the GMoM to obtain a robust aggregate, we would like to mention that alternative frameworks for this problem exist, e.g. distributionally robust optimization \cite{blanchet2024distributionallyrobustoptimizationrobust, kuhn2024distributionallyrobustoptimization}.

\section{Robust score matching for contaminated data}
\label{sec:robustsm}
This section introduces a generalization of the score matching estimator $\hat{\bs{\theta}}$ from \eqref{eqn:emp_sm} and investigates its robustness against contamination. 

\subsection{Contamination assumptions}\label{subsec:contamination_assumptions}
In the classical Tukey-Huber contamination model \citep[Sect. 2.2]{MaronnaRobustStatistics}, the observed vector $\mathbf{y}\in\mathbb{R}^p$ equals $\mathbf{y}=(\mathbf{I} - \mathbf{B})\mathbf{x} + \mathbf{B}\mathbf{z}$, where $\mathbf{x}$ is the uncorrupted observation, $\mathbf{z}$ is a random contamination vector, $\mathbf{I}$ is the $p\times p$ identity matrix, and $\mathbf{B}$ is a diagonal matrix, either being $\mathbf{I}$ with probability $\varepsilon>0$ or the zero-matrix otherwise. In a data frame where rows are observations, under the Tukey-Huber model any row is either corrupted or not, and thus this is a form of \emph{rowwise} corruption.  

In this paper, we consider rowwise contamination, however, we do not  require that the contamination occurs at random like in the Tukey-Huber model. Instead, we assume that a proportion $\varepsilon$ of the rows could have been altered arbitrarily. This includes \emph{adversarial} contamination by an intelligent attacker; see, e.g., \citet{BhattAdversarialContamination}.
We note that yet other forms of corruption could be considered in future work; compare, e.g., the \emph{cellwise} contamination treated by  \citeauthor{AlqallafCellwiseCorruption}, \citeyear{AlqallafCellwiseCorruption}.

\subsection{A robust estimator based on the GMoM}

The classical score matching estimator $\hat{\bs{\theta}}$ from \eqref{eqn:emp_sm} minimizes $\tfrac{1}{2}\bs{\theta}^\top \overline{\mathbf{\Gamma}}(\mathbf{X}) \bs{\theta} - \bs{\theta}^\top \overline{\mathbf{g}}(\mathbf{X})$. We propose to replace $\overline{\mathbf{\Gamma}}(\mathbf{X})$ and $\overline{\mathbf{g}}(\mathbf{X})$ with a more robust version using the GMoM. In symbols, we set
\begin{equation}\label{eq:estimator:def_parts}
    \begin{aligned}
    \hat{\mathbf{\Gamma}}_K(\mathbf{X}) &:= \gmom{K}{\mathbf{\Gamma}(\xvect{1}),\ldots, \mathbf{\Gamma}(\xvect{n})}, \\
    \hat{\mathbf{g}}_K(\mathbf{X}) &:= \gmom{K}{\mathbf{g}(\xvect{1}),\ldots, \mathbf{g}(\xvect{n})}.
\end{aligned} 
\end{equation}
Note that when applying the GMoM each $\mathbf{\Gamma}(\xvect{i})\in\mathbb{R}^{r\times r}$ is interpreted as a vector in $\mathbb{R}^{r^2}$. When the parameter $K$ for the number of blocks equals one, $\hat{\mathbf{\Gamma}}_K(\mathbf{X})$ and $\hat{\mathbf{g}}_K(\mathbf{X})$ reduce to $\overline{\mathbf{\Gamma}}(\mathbf{X})$ and $\overline{\mathbf{g}}(\mathbf{X})$. For $K>1$, the equal weights $1/n$ are replaced by non-negative weights that sum to one in which block means that contain outliers are downweighted, as shown in~\eqref{eqn:GmedIteration}.
We then propose the estimator 
\begin{align}\label{eqn:robust_sm_def}
    \hat{\bs{\theta}}(K) := \underset{\bs{\theta}\in\Omega}{\argmin}\; \tfrac{1}{2}\bs{\theta}^\top \hat{\mathbf{\Gamma}}_K(\mathbf{X}) \bs{\theta} - \bs{\theta}^\top \hat{\mathbf{g}}_K(\mathbf{X}),
\end{align}
which exists uniquely if and only if $\hat{\mathbf{\Gamma}}_K(\mathbf{X})$
is positive definite, in which case $\hat{\bs{\theta}}(K)=\hat{\mathbf{\Gamma}}_K^{-1}(\mathbf{X})\hat{\mathbf{g}}_K(\mathbf{X})$.  
In the classical problem \eqref{eqn:emp_sm}, $\overline{\mathbf{\Gamma}}(\mathbf{X})$ looks very similar to a sample covariance matrix, which would be almost surely positive definite when the sample size $n$ exceeds the ambient dimension $m$ \citep{eaton:perlman:1973}.
Similarly, a sufficient sample size guarantees the positive definiteness of $\hat{\mathbf{\Gamma}}_K(\mathbf{X})$ under mild regularity conditions on the sufficient statistic $\mathbf{t}$, as detailed in the appendix. It is this guarantee of positive definiteness that stems from the use of the geometric median over conceptually and computationally simpler methods like the componentwise median. 

\subsection{A first robustness guarantee under contamination}\label{subsec:robust_guarant_basic_sm}

We now show that the robust score matching estimator from \eqref{eqn:robust_sm_def} consistently estimates the true parameter $\bs{\theta}_0$ even if a part of the observations are contaminated as described in Section~\ref{subsec:contamination_assumptions}. We begin by deriving a concentration result of the GMoM around the population mean under corruption, which is also useful when we consider sparse graphical models in the next section.

\begin{theorem}\label{thm:gmom_concentration}
    Let $\xvect{1},\ldots,\xvect{n}\in\mathbb{R}^p$ be independent samples from a $p$-dimensional distribution with mean $\bs{\mu}$ and variance $\mathbf{\Sigma}$. Fix a confidence level of $0<\delta\leq 1$.  We allow for up to $(\lfloor 17\cdot\log(1/\delta)\rfloor + 1)\tau$ samples to be arbitrarily corrupted, where $0\leq\tau< 1/2$. There exists functions $k(\tau)=\mathcal{O}(1/(\tfrac{1}{2} - \tau)^{2})$ and $c(\tau)=\mathcal{O}(1/(\tfrac{1}{2} - \tau)^{2.5})$ as $\tau\to\tfrac{1}{2}$ such that when the number of blocks $K$ defined as $K=K(\delta, \tau) :=\lfloor k(\tau)\cdot \log(1/\delta)\rfloor + 1$ satisfies $K\leq n/2$, it holds that  \begin{equation}\label{eqn:thm_gmom_concentration}
        \prob\biggl( \|\gmom{K}{\xvect{1},\ldots,\xvect{n}} - \bs{\mu}\|_2 >
        c(\tau)\sqrt{\log\left(\frac{4}{(1-\tau)^2}\frac{1}{\delta}\right)\frac{\tr(\mathbf{\Sigma})}{n}}\biggr) \leq \delta.
    \end{equation}
\end{theorem}

To interpret Theorem \ref{thm:gmom_concentration}, it is helpful to consider the complementary statement of \eqref{eqn:thm_gmom_concentration}. It reads that with probability at least $1-\delta$, the GMoM approximates $\bs{\mu}$ correctly up to some bound $B(n,\delta,\tau)$. To illustrate how this can be used, assume that the number of corrupt samples $n_c$ grows with $n$ but is $o(n)$. For some fixed $\tau_0$, one can set $\delta:=\exp(-(\lceil n_c / \tau_0 \rceil - 1)/17)$ to satisfy the assumptions of the theorem, resulting in $K=o(n)$. By the assumption on $n_c$, both $\delta$ and $B(n, \delta, \tau_0)$ are $o(1)$ as $n\to\infty$. Thus, the GMoM converges in probability to $\bs{\mu}$ as $n\to\infty$ in the considered setting:

\begin{corollary}
    Let $\xvect{1},\ldots,\xvect{n}\in\mathbb{R}^p$ be independent samples from an exponential family distribution with parameter $\bs{\theta}_0\in\mathbb{R}^r$. Assume the data generating model satisfies the mild regularity assumptions listed in the appendix. Allow for up to $n_c$ samples to be arbitrarily corrupted, where $n_c=o(n)$. Then, there exists a sequence $K=K(n_c)$ such that the robust score matching estimator $\hat{\bs{\theta}}(K(n_c))$ from \eqref{eqn:robust_sm_def} converges to $\bs{\theta}_0$ in probability as $n\to\infty$. 
\end{corollary}

The concentration statement \eqref{eqn:thm_gmom_concentration} is based on the work of \citet[Cor.~4.1 \& Rem. 3.1]{MinskerGeometricMoM}. The theorem shares traits with results in the literature: a logarithmic relation between $1/\delta$ and $K$ \citep{LugosiMeanEstimationHeavyTails} and between $1/\delta$ and the corruption $\tau$ \citep{LaforgueMoMcorrupt}.
The range of the corruption parameter $\tau$ between $0$ and $1/2$ reflects the high breakdown point of $1/2$ of the geometric median, cf.~(\textbf{R3}). Concretely, the assumptions of the theorem ensure that at most $\tau K < K/2$ block means are corrupted, as detailed in the proof. For the Tukey median for instance, we would expect $\tau < 1/(p+1)$.

\subsection{Choice of number of blocks \texorpdfstring{$K$}{K}}\label{subsec:choice_K}
Choosing the number of blocks $K$ is a trade-off between robustness, bias and variance. As the number of blocks increases, the GMoM becomes more robust since the breakdown point of the GMoM is equal to $\lfloor \tfrac{1}{2}(K+1) \rfloor/n$, which grows with $K$. The effect that increasing $K$ has on the variance is problem dependent. For Gaussian location estimation, the maximal choice $K=n$ has higher asymptotic variance than the mean $K=1$ as shown in \citet{BrownStatisticalUsesGMed}. In heavy tailed scenarios on the other hand, the GMoM has relatively light tails as shown in Theorem \ref{thm:gmom_concentration}, which indicates that it can have lower variance than the sample mean. The bias of the GMoM also depends on the problem. The GMoM is an unbiased location estimator for any $K$ when the underlying distribution is centrally symmetric, i.e., when  $\mathbf{x} - \ex{\mathbf{x}}$ and $\ex{\mathbf{x}}-\mathbf{x}$ have the same distribution \citep{serfling2006multivariate}. In general however, the geometric median is a biased estimator for the population mean, making the GMoM biased as well.  Still, the GMoM typically has small bias for small $K$, as the central limit theorem implies that the block means are approximately Gaussian and thus centrally symmetric. For large $K$ the bias will generally be larger as the GMoM approaches the geometric median at $K=n$.

A general choice of $K$ when a proportion $\varepsilon$ of samples is corrupted should comfortably exceed the breakdown point for robustness, but not be too large in order to avoid bias. This reasoning is supported by a simulation study in the appendix. Since the breakdown point is exceeded when $K\geq 2\varepsilon$, we propose $K:=  4\varepsilon n $ as a compromise. This choice works well empirically as shown in Section \ref{sec:experiments}. If $4\varepsilon n$ is not an integer, is smaller than one, or is greater than $n$, $K$ is chosen to be the nearest admissible integer.

\section{Application to high-dimensional graphical modeling}\label{sec:highDimGraphModel}

As an application of special interest, we consider a general  pairwise interaction model given by \begin{equation}\label{eqn:pw_model_density}
    p(x|\mathbf{\Theta}, \bs{\eta}) := \exp\bigg(- \sum_{1\leq i\leq j\leq m}\Theta_{ij} t_{ij}(x_i, x_j) - \sum_{i=1}^m \eta_i t_i(x_i) - a(\mathbf{\Theta}, \bs{\eta})\bigg),\quad \xrandvect\in\mathcal{X},
\end{equation} where the domain $\mathcal{X}$ can be $\mathbb{R}^m$ or have boundaries like in Example \ref{example:sqr}. Let $\mathbf{X}=\{\xvect{1},\ldots,\xvect{n}\}$ be an i.i.d. sample from \eqref{eqn:pw_model_density}. Score matching for pairwise interaction models simplifies structurally when the dummy variables $\Theta_{ji}:= \Theta_{ij}$ for $j < i$ are introduced;  for example, $\mathbf{\Gamma}(\xrandvect)$ is block-diagonal \citep{YuScoreMatchNonNeg}. To apply the theory from \citet{YuScoreMatchNonNeg}, this section assumes that $\mathbf{\Gamma}(\xrandvect)$ and $\mathbf{g}(\xrandvect)$ are derived for the (extended) square parameter matrix $\mathbf{\Theta}=(\Theta_{ij})$. We abbreviate the pair of $\mathbf{\Theta}$ together with the parameter vector $\bs{\eta}$ by a single $r$-dimensional parameter $\bs{\theta}$.

Motivated by applications such as gene regulatory networks, we focus on the case that the dimension $m$ is large, most $\Theta_{ij}$ are zero \citep{OhSparseGeneRegulatoryNetworks}, and the sample size $n$ is smaller than the dimension $m$
\citep{ChuHighDimGGMPractical}.
To incorporate the sparsity assumption, we include an $\ell_1$-regularization penalty in the objective function. For $n < m$, we follow \citet{YuScoreMatchNonNeg} and include a diagonal multiplier that ensures that the Gram matrix is positive definite. We thus propose the following estimator:

\begin{definition}\label{def:estimator} 
Using $\hat{\mathbf{\Gamma}}_K$ and $\hat{\mathbf{g}}_K$ from \eqref{eq:estimator:def_parts}, define for $\beta, \lambda >0$ with $\hat{\bs{\Gamma}}_{K;\beta}:= \hat{\mathbf{\Gamma}}_K + \beta\cdot\diag(\hat{\mathbf{\Gamma}}_K)$,
\begin{align}\label{eqn:full_estimator_def}
\hspace{-0.2cm}\hat{\bs{\theta}}(K, \beta, \lambda):= \underset{\bs{\theta}\in\Omega}{\argmin}\frac{1}{2}\bs{\theta}^\mathsf{T} \hat{\bs{\Gamma}}_{K;\beta}\bs{\theta} + \hat{\mathbf{g}}_K^\mathsf{T} \bs{\theta} + \lambda \|\bs{\theta}\|_1.\end{align}
\end{definition}

Computationally, this estimator is attractive. First, $\mathbf{\Gamma}(\xvect{i})$ and $\mathbf{g}(\xvect{i})$ are assembled for $i=1,\ldots, n$. This needs $\mathcal{O}(nmr^2)$ operations under the assumption that $\partial_i T$ is evaluated in constant time. Since the number of parameters $r$ is $\mathcal{O}(m^2)$, the complexity is $\mathcal{O}(nm^5)$, although factoring in the block-diagonal structure of $\mathbf{\Gamma}(\xrandvect)$ reduces this to $\mathcal{O}(nm^4)$ and even $\mathcal{O}(nm^3)$ in very symmetric models like the square root graphical model \citep{YuScoreMatchNonNeg}.
Next, the GMoMs $\hat{\mathbf{\Gamma}}_K(\mathbf{X})$ and $\hat{\mathbf{g}}_K(\mathbf{X})$ are computed iteratively, where each iteration requires $\mathcal{O}(r^2K)=\mathcal{O}(m^4 K)$ operations (only $\mathcal{O}(m^3 K)$ when factoring in the block structure of $\mathbf{\Gamma}(\mathbf{x})$). The actual $\ell_1$-regularized optimization problem can be solved by iterative methods like coordinate descent \citep{FriedmannCoordinateDescent}. 

To treat the estimator from \eqref{eqn:full_estimator_def} theoretically, we introduce the following notation and definitions:

\begin{definition} \label{def:score_m_success_conditions}
Let $\bs{\theta}_0=(\mathbf{\Theta}_0, \bs{\eta}_0)$ be the unknown true parameter and $\mathbf{\Gamma}_0 := \ex[\bs{\theta}_0]{\mathbf{\Gamma}(\xrandvect)}$. Define $$d_{\bs{\theta}_0} := \underset{j=1,\ldots, m}{\max} \bigl(\#\{i : (\mathbf{\Theta}_0)_{ij}\neq 0\} + 1\{(\bs{\eta}_0)_{j} \neq 0\}\bigr).$$ 
Let $c_{\bs{\theta}_0}:=\vertiii{\mathbf{\Theta}_0}_{\infty,\infty}$.
Write $S(\bs{\theta}):=\{i : \bs{\theta}_i\neq 0\}$ for the support of a parameter vector. Abbreviate $S_0 := S(\bs{\theta}_0)$.
Further, if $\mathbf{\Gamma}_{0,S_0S_0}$ is invertible, set \begin{align*}c_{\mathbf{\Gamma}_0}&:=\vertiii{(\mathbf{\Gamma}_{0,S_0S_0})^{-1}}_{\infty, \infty}\\
I_{S_0} &:= \vertiii{\mathbf{\Gamma}_{0, S_0^c S_0}(\mathbf{\Gamma}_{0, S_0S_0}^{-1})}_{\infty, \infty}.\end{align*}

Finally, we say $\mathbf{\Gamma}_0$ satisfies the \emph{irrepresentability condition} with incoherence parameter $\alpha\in(0,1]$ and edge set $S_0$, if $I_{S_0}\leq (1 - \alpha)$.
\end{definition}

The following theorem shows concentration of $\hat{\bs{\theta}}(K, \beta, \lambda)$ around the true parameter $\bs{\theta}_0$ with high probability even under contamination, extending the work of \citet{YuScoreMatchNonNeg} and \citet{LinHighDimScoreMatching}.

\begin{theorem}\label{thm:success_guarantee_theta_hat}
    Let $\xvect{1},\ldots,\xvect{n}\in\mathbb{R}^m$ be i.i.d. samples from a pairwise interaction model with parameter $\bs{\theta}_0$. Assume that $\mathbf{\Gamma}_0$ satisfies the irrepresentability condition with parameter $\alpha$ and edge set $S_0$. Further, suppose $\mathbf{\Sigma}_{\mathbf{\Gamma}_0} := \Var_{\bs{\theta}_0}(\mathbf{\Gamma}(\xrandvect))$ and $\mathbf{\Sigma}_{\mathbf{g}_0} := \Var_{\bs{\theta}_0}(\mathbf{g}(\xrandvect))$ exist with $\tr(\mathbf{\Sigma}_{\mathbf{\Gamma}_0}) >0$.
    
     Fix a confidence level $0 < \delta \leq 1$. We allow up to $n_c :=\tau (\lfloor 17\cdot\log(1/\delta)\rfloor + 1)$ samples being arbitrarily corrupted, with $0\leq\tau< 1/2$. Let $K = K(\delta, \tau)$ and $c(\tau)$ be as in Theorem \ref{thm:gmom_concentration}. Let \begin{equation}\label{eqn:thm_theta_hat_guarant_diag_mult_req}0\leq \beta \leq \frac{1}{1 + (\|\mathbf{\Gamma}_0\|_2/\sqrt{2\tr(\mathbf{\Sigma}_{\mathbf{\Gamma}_0})})\,\sqrt{n/K}}.\end{equation} Finally, with constants and notation from Definition \ref{def:score_m_success_conditions}, if \begin{align}\label{eqn:thm_theta_hat_guarant_n_growth}
        \hspace{-0.2cm}n & > \left(\frac{24d_{\bs{\theta}_0}c_{\mathbf{\Gamma}_0}c(\tau)}{\alpha}\right)^2 \log\left(\frac{4}{(1-\tau)^2}\frac{1}{\delta}\right) \tr(\mathbf{\Sigma}_{\mathbf{\Gamma}_0}) , \\\label{eqn:thm_theta_hat_guarant_lambda_growth}
        \hspace{-0.2cm}\lambda & > \frac{6 c(\tau)(2 - \alpha)}{\alpha}\sqrt{\log\left(\frac{4}{(1-\tau)^2}\frac{1}{\delta}\right)\frac{1}{n}}\cdot \\\notag &\quad\max\left(2 c_{\bs{\theta}_0} \sqrt{\tr(\mathbf{\Sigma}_{\mathbf{\Gamma}_0})}, \sqrt{\tr(\mathbf{\Sigma}_{g_0})}\right),
    \end{align}
    with probability at least $1 - 2\delta$, the estimator $\hat{\bs{\theta}}(K, \beta, \lambda)$ is unique with $S(\hat{\bs{\theta}}(K, \beta, \lambda)) \subset S_0$ and \begin{equation}\label{eqn:thm_theta_hat_guarant_result}\|\hat{\bs{\theta}}(K, \beta, \lambda) - \bs{\theta}_0\|_\infty\leq\frac{c_{\mathbf{\Gamma}_0}}{2 - \alpha}\lambda.\end{equation} 
\end{theorem}

Theorem \ref{thm:success_guarantee_theta_hat} guarantees that with high probability the maximal difference between the estimated model parameters $\hat{\bs{\theta}}(K, \beta, \lambda)$ and the true parameter $\bs{\theta}_0$ is small, and that any non-zero interaction in the model induced by $\hat{\bs{\theta}}(K, \beta, \lambda)$ is also present in the true model. To illustrate the implications of the theorem, like in Section~\ref{subsec:robust_guarant_basic_sm}, consider $n_c=o(n)$, $\tau:=\tau_0$ and $\delta:=\exp(-(\lceil n_c / \tau_0 \rceil - 1)/17)$. Then, $K=o(n)$ and $\beta=o(1)$ as $n\to\infty$. Since $\log(4/(1-\tau_0)^2\cdot 1/\delta)=o(n)$, the requirement \eqref{eqn:thm_theta_hat_guarant_n_growth} is satisfied for large $n$, and the lower bound in \eqref{eqn:thm_theta_hat_guarant_lambda_growth} allows a choice $\lambda=o(1)$. By \eqref{eqn:thm_theta_hat_guarant_result}, $\hat{\bs{\theta}}(K, \beta, \lambda)$ converges to $\bs{\theta}_0$ in probability as $n\to\infty$.

Theorem \ref{thm:success_guarantee_theta_hat} reads similarly to the theorems in \citep[Sect. 6]{YuScoreMatchNonNeg}. However, since Theorem \ref{thm:success_guarantee_theta_hat} does not assume an underlying Gaussian distribution, that is possibly truncated, the bounds on $\beta, n$ and $\lambda$ depend on $\Gammanaught$ and $\mathbf{g}_0$ explicitly.

\section{Numerical experiments}
\label{sec:experiments}

\subsection{Simulation study}
\label{subsec:simulations}

\begin{figure*}[t]
    \centering
    \includegraphics[scale=0.85]{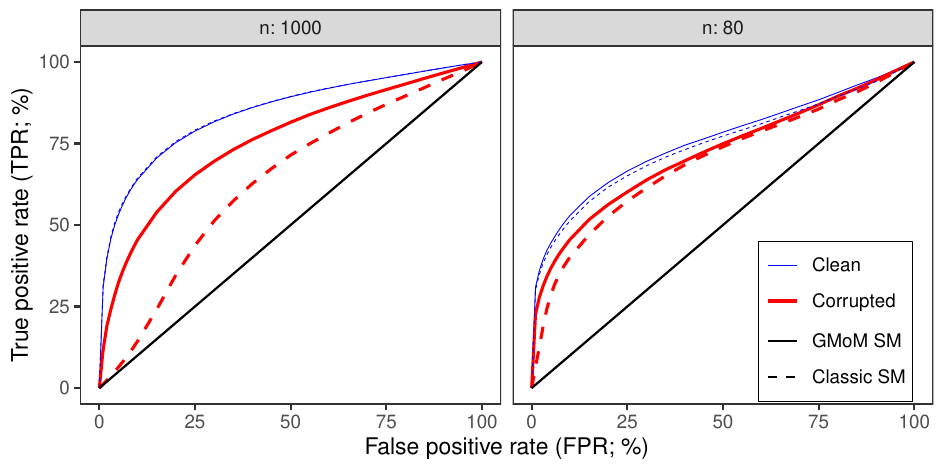}
    \caption{ROC curves for support recovery in the square root model. The pointwise uncertainty in TPR is at most $\pm 0.75\%$ based on $100$ (left) and $500$ simulations (right).}
    \label{fig:sqr_support_recovery}
    

\end{figure*}

\begin{figure*}[t]
    \centering
    
    \begin{subfigure}{0.5\textwidth}
    \includegraphics[width=0.9\linewidth]
    {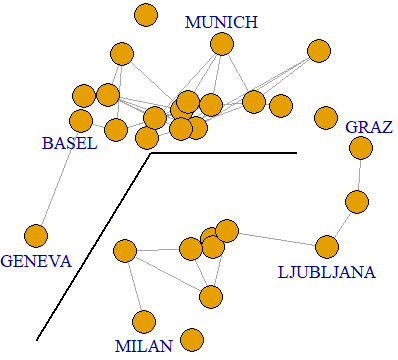} 
    \end{subfigure}\hfill
    \begin{subfigure}{0.5\textwidth}
    \includegraphics[width=0.9\linewidth]
    {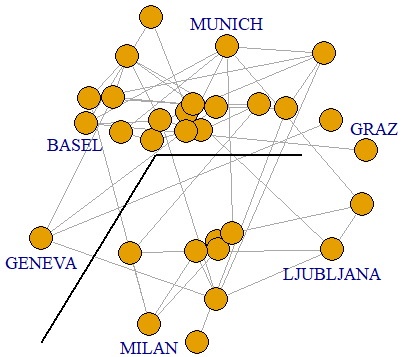}
    \end{subfigure}

    \caption{Precipitation dependencies (grey lines) learned from uncorrupted (left) and contaminated data (right). Black lines sketch the location of some of the highest mountains of the Alps.}
    \label{fig:application_graphs}
\end{figure*}

We apply the estimator $\hat{\bs{\theta}}(K, \beta, \lambda)$ from Definition \ref{def:estimator} to simulated data from square root graphical models, under scenarios that include rowwise corruption. The `standard' regularized score matching estimator,
corresponding to $K=1$ block, serves as a baseline. Performance is judged by how well the zero structure of $\mathbf{\Theta}$ is learned, assessed through receiver operator characteristic (ROC) curves in Figure \ref{fig:sqr_support_recovery}. The simulations run within a few hours on a personal laptop. Experiments on Gaussian data, one of the classic graphical models classes, lead to similar results and are contained in the appendix.

\textbf{Data-generating model:} As in \cite{YuScoreMatchNonNeg}, we considered a $m=100$ dimensional model with either $n=1000$ or $n=80$ samples, the latter acting as a high-dimensional scenario. Additional choices for $n$ are considered in the appendix. The interaction matrix $\mathbf{\Theta}$ was determined by first selecting a graph on $m$ vertices uniformly from the set of all graphs with $\kappa$ edges (one of the variants of the Erdős–Rényi graph model), and then drawing the edge strength $\Theta_{ij}$ from $\pm\text{Unif}(0.5,1)$. The ratio $\kappa/n$ was kept constant and set to $1/2$ to have an average node degree of $m/10$ for $n=1000$. Each ROC curve reports average ROCs \citep{FawcettIntroductionToROC} for $10$ randomly chosen $\mathbf{\Theta}$. The location-like parameter $\bs{\eta}$ was randomly drawn from $\{0,0.5,-0.5\}$. 

\textbf{Contamination details:} For the simulations with contamination, $5\%$ of data rows were replaced by independent Pareto draws. The Pareto scale parameter was set to the respective column mean and the shape parameter to $1$, which ensures that most corrupted values are similar to the uncorrupted values and, due to the heavy tails of the Pareto, a small portion are strong outliers. Results under different contamination settings are reported in the appendix.

\textbf{Hyperparameter tuning:} The number of blocks was set to $K:= 4\cdot0.05\,n=n/5$ as discussed in Section \ref{subsec:choice_K}. In simulations with $5\%$ contamination, $\hat{\bs{\theta}}(K, \beta, \lambda)$ is thus adapted to the actual corruption amount, while in the uncorrupted case it represents a conservative block size choice. The baseline from \citet{YuScoreMatchNonNeg} represents the opposite: it has $K=1$ by definition, making it adapted to the uncontaminated simulations, but underestimating the corruption amount otherwise. The diagonal multiplier $\beta$ was set to $0$ for $n=1000$, and for $n=80$ to the upper bound in Theorem \ref{thm:success_guarantee_theta_hat} with $\Gammanaught$ being estimated from uncorrupted data (yielding $\beta\approx 0.01$). The upper bound was chosen since \citet{YuScoreMatchNonNeg} experimentally found this aided support recovery. The regularization parameter $\lambda$ was varied to obtain a ROC curve. The weights are set to $\mathbf{h}(\xrandvect):= \xrandvect^{3/2}$, a choice found to be favorable in \citep{YuScoreMatchNonNeg}.

\textbf{Interpretation:}
Figure \ref{fig:sqr_support_recovery} shows that the GMoM procedure is on par with the baseline in terms of support recovery on uncorrupted data and outperforms the baseline on contaminated data. A pointwise $95\%$ bootstrap confidence band around the curves has a maximal width of $\pm 0.75\%$, implying that this conclusion is statistically significant. For the high-dimensional experiment $n=80$, the effect is less pronounced and the difference under contamination only significant for the lower end of the FPR spectrum. This is due to the small sample size and increased sparsity, making the system already very noisy even without corruption.

\subsection{Data on precipitation across the Alps}
\label{subsec:dataexample}

Consider the task of learning how precipitation at different weather stations in central Europe is related. We model the dependence between monthly total precipitation at $m=30$ stations using the European Climate Assessment \& Dataset (ECA\&D) \citep{klein2002daily} from \url{www.ecad.eu}. The records of the stations share a span of $87$ years. Data from November, January and March was used to obtain similar precipitation distributions, which are roughly independent due to the one-month gap. This leads to $n=3\cdot 87 = 261$ samples. The data clearly is non-Gaussian, for example due to positivity. Instead, inspired by the marginal distributions, the square root graphical model is chosen. To learn the model, $\lambda$ is tuned such that the graph contains $45$ edges to get an average node degree of $3$, which roughly equals the number of geographical neighbors of the average station. The diagonal multiplier $\beta$ is not needed since $n \gg m$. The graph learned by the baseline score matching approach ($K=1$) 
is shown on the left of Figure \ref{fig:application_graphs}.

While there is no direct ground truth information on the precipitation dependence, the graph from the uncorrupted data has the notable feature that it only connects stations within the same geographical neighborhood. Moreover, no edge crosses the Alps mountains sketched by black lines in Figure \ref{fig:application_graphs}. This is expected since high mountains act as a barrier for clouds. According to the latter physical consideration, edges crossing the Alps are false discoveries. The corresponding false discovery rate, termed Alps-FDR, is used to judge a learned precipitation network. No edges crossing the Alps yields the optimal $0\%$ Alps-FDR.   

Now, as our experiment, we alter the data through random contamination as in Section \ref{subsec:simulations}. A graph learned under $5\%$ contamination using the non-robust the baseline ($K=1$) is shown on the right of Figure \ref{fig:application_graphs}. It is noticeably more noisy and connects stations further away. Its Alps-FDR is $19\%$, meaning that roughly every fifth connection is considered unrealistic. In a Monte Carlo simulation with $100$ respective runs, different proportions $\varepsilon$ of the sample were contaminated and the Alps-FDR of the baseline $K=1$ compared with $K=4\varepsilon n$. Additionally, a Gaussian graphical model (GGM), arguably the most studied graphical model, was fit to the data, knowing that the Gaussianity assumption was violated. Results are reported in Table \ref{tab:application_result_table} with $95\%$ bootstrap confidence intervals. It is evident that the GMoM version with $K=4\varepsilon n$ has the best Alps-FDR in every corruption scenario. For comparison, chosing a graph uniformly at random from all graphs with $45$ edges has an Alps-FDR of roughly $42\%$.    

\begin{table}[t]
\renewcommand{\arraystretch}{1.05}
    \centering
    \caption{The Alps-FDR for different corruption proportions computed from $100$ Monte Carlo runs.}
    \label{tab:application_result_table}
    \vspace{\the\baselineskip}
    \begin{tabular}{S[table-format=2]S[table-format=2.1 \pm 1.1]S[table-format=2.1 \pm 1.1]S[table-format=2.1 \pm 1.1]}
    \hline
    {Corr. ($\%$)} & {GMoM SM} & {Classic SM} & {GGM}\\\hline
         1&   2.4\pm 0.5 &  4.2\pm 0.8 & 12.8\pm 1.6\\
         5&   7.0\pm 0.8 & 13.5\pm 1.0 & 26.9\pm 1.7\\
         10&  8.1\pm 0.8 & 17.8\pm 1.0 & 31.7\pm 1.3\\
         20&  10.5\pm 1.0 & 22.4\pm 1.1 & 38.7\pm 1.4
    \end{tabular}  
\end{table}

\section{Conclusion and future work}

This paper introduces a robust score matching estimator that utilizes the geometric median of means to circumvent existence issues that result from more naive robustification approaches. Theoretical guarantees and empirical evidence demonstrate our estimator's ability to recover the dependence structure of a pairwise interaction model, even when a portion of the observations is contaminated. In the presented numerical experiments on uncorrupted data, the dependence recovery was on par with that of the classical regularized score matching estimator from \citet{YuScoreMatchNonNeg}. 

An interesting topic for future work is to further examine the optimal choice of the number of blocks $K$. Evidently, there is a trade-off between bias and variance inherent to score matching; especially for asymmetric, heavy tailed distributions. Neither bias nor variance of the geometric median of means in this scenario seems to be well understood. How the trade-off is influenced by contamination, possibly also in different forms such as cellwise contamination, is another open problem. Additionally, the concentration guarantee in Theorem \ref{thm:success_guarantee_theta_hat} could likely be improved if one were able to refine the interplay between $\|\cdot\|_2$ from the concentration of the geometric median and $\|\cdot\|_1$ from the $\ell_1$-regularization.

\section*{Acknowledgements}
We acknowledge the data providers in the ECA\&D project \citep{klein2002daily}. Data and metadata available at \url{https://www.ecad.eu}. This paper is supported by the DAAD programme Konrad Zuse Schools of Excellence in Artificial Intelligence, sponsored by the Federal Ministry of Education and Research. Further, this work has been funded by the German Federal Ministry of Education and Research and the Bavarian State Ministry for Science and the Arts. The authors of this work take full responsibility for its content. This project has received funding from the European Research Council (ERC) under the European Union’s Horizon 2020 research and innovation programme (grant agreement No 883818).

\bibliography{bib}


\begin{thebibliography}{38}
\ifx \bisbn   \undefined \def \bisbn  #1{ISBN #1}\fi
\ifx \binits  \undefined \def \binits#1{#1}\fi
\ifx \bauthor  \undefined \def \bauthor#1{#1}\fi
\ifx \batitle  \undefined \def \batitle#1{#1}\fi
\ifx \bjtitle  \undefined \def \bjtitle#1{#1}\fi
\ifx \bvolume  \undefined \def \bvolume#1{\textbf{#1}}\fi
\ifx \byear  \undefined \def \byear#1{#1}\fi
\ifx \bissue  \undefined \def \bissue#1{#1}\fi
\ifx \bfpage  \undefined \def \bfpage#1{#1}\fi
\ifx \blpage  \undefined \def \blpage #1{#1}\fi
\ifx \burl  \undefined \def \burl#1{\textsf{#1}}\fi
\ifx \doiurl  \undefined \def \doiurl#1{\url{https://doi.org/#1}}\fi
\ifx \betal  \undefined \def \betal{\textit{et al.}}\fi
\ifx \binstitute  \undefined \def \binstitute#1{#1}\fi
\ifx \binstitutionaled  \undefined \def \binstitutionaled#1{#1}\fi
\ifx \bctitle  \undefined \def \bctitle#1{#1}\fi
\ifx \beditor  \undefined \def \beditor#1{#1}\fi
\ifx \bpublisher  \undefined \def \bpublisher#1{#1}\fi
\ifx \bbtitle  \undefined \def \bbtitle#1{#1}\fi
\ifx \bedition  \undefined \def \bedition#1{#1}\fi
\ifx \bseriesno  \undefined \def \bseriesno#1{#1}\fi
\ifx \blocation  \undefined \def \blocation#1{#1}\fi
\ifx \bsertitle  \undefined \def \bsertitle#1{#1}\fi
\ifx \bsnm \undefined \def \bsnm#1{#1}\fi
\ifx \bsuffix \undefined \def \bsuffix#1{#1}\fi
\ifx \bparticle \undefined \def \bparticle#1{#1}\fi
\ifx \barticle \undefined \def \barticle#1{#1}\fi
\bibcommenthead
\ifx \bconfdate \undefined \def \bconfdate #1{#1}\fi
\ifx \botherref \undefined \def \botherref #1{#1}\fi
\ifx \url \undefined \def \url#1{\textsf{#1}}\fi
\ifx \bchapter \undefined \def \bchapter#1{#1}\fi
\ifx \bbook \undefined \def \bbook#1{#1}\fi
\ifx \bcomment \undefined \def \bcomment#1{#1}\fi
\ifx \oauthor \undefined \def \oauthor#1{#1}\fi
\ifx \citeauthoryear \undefined \def \citeauthoryear#1{#1}\fi
\ifx \endbibitem  \undefined \def \endbibitem {}\fi
\ifx \bconflocation  \undefined \def \bconflocation#1{#1}\fi
\ifx \arxivurl  \undefined \def \arxivurl#1{\textsf{#1}}\fi
\csname PreBibitemsHook\endcsname

\bibitem[\protect\citeauthoryear{Alqallaf et~al.}{2009}]{AlqallafCellwiseCorruption}
\begin{barticle}
\bauthor{\bsnm{Alqallaf}, \binits{F.}},
\bauthor{\bsnm{Van~Aelst}, \binits{S.}},
\bauthor{\bsnm{Yohai}, \binits{V.J.}},
\bauthor{\bsnm{Zamar}, \binits{R.H.}}:
\batitle{Propagation of outliers in multivariate data}.
\bjtitle{Ann. Statist.}
\bvolume{37}(\bissue{1}),
\bfpage{311}--\blpage{331}
(\byear{2009})
\end{barticle}
\endbibitem

\bibitem[\protect\citeauthoryear{Bhatt et~al.}{2022}]{BhattAdversarialContamination}
\begin{bchapter}
\bauthor{\bsnm{Bhatt}, \binits{S.}},
\bauthor{\bsnm{Fang}, \binits{G.}},
\bauthor{\bsnm{Li}, \binits{P.}},
\bauthor{\bsnm{Samorodnitsky}, \binits{G.}}:
\bctitle{Minimax m-estimation under adversarial contamination}.
In: \bbtitle{Proceedings of the 39th International Conference on Machine Learning}.
\bsertitle{Proceedings of Machine Learning Research},
vol. \bseriesno{162},
pp. \bfpage{1906}--\blpage{1924}.
\bpublisher{PMLR},
\blocation{New York, USA}
(\byear{2022})
\end{bchapter}
\endbibitem

\bibitem[\protect\citeauthoryear{Blanchet et~al.}{2024}]{blanchet2024distributionallyrobustoptimizationrobust}
\begin{botherref}
\oauthor{\bsnm{Blanchet}, \binits{J.}},
\oauthor{\bsnm{Li}, \binits{J.}},
\oauthor{\bsnm{Lin}, \binits{S.}},
\oauthor{\bsnm{Zhang}, \binits{X.}}:
Distributionally Robust Optimization and Robust Statistics
(2024).
\url{https://arxiv.org/abs/2401.14655}
\end{botherref}
\endbibitem

\bibitem[\protect\citeauthoryear{Brown}{1983}]{BrownStatisticalUsesGMed}
\begin{barticle}
\bauthor{\bsnm{Brown}, \binits{B.M.}}:
\batitle{Statistical uses of the spatial median}.
\bjtitle{J. Roy. Statist. Soc. Ser. B}
\bvolume{45}(\bissue{1}),
\bfpage{25}--\blpage{30}
(\byear{1983})
\end{barticle}
\endbibitem

\bibitem[\protect\citeauthoryear{Chu et~al.}{2009}]{ChuHighDimGGMPractical}
\begin{barticle}
\bauthor{\bsnm{Chu}, \binits{J.-h.}},
\bauthor{\bsnm{Weiss}, \binits{S.T.}},
\bauthor{\bsnm{Carey}, \binits{V.J.}},
\bauthor{\bsnm{Raby}, \binits{B.A.}}:
\batitle{A graphical model approach for inferring large-scale networks integrating gene expression and genetic polymorphism}.
\bjtitle{BMC Systems Biology}
\bvolume{3},
\bfpage{1}--\blpage{9}
(\byear{2009})
\end{barticle}
\endbibitem

\bibitem[\protect\citeauthoryear{Donoho and Gasko}{1992}]{DonohoBreakdownP}
\begin{barticle}
\bauthor{\bsnm{Donoho}, \binits{D.L.}},
\bauthor{\bsnm{Gasko}, \binits{M.}}:
\batitle{Breakdown properties of location estimates based on halfspace depth and projected outlyingness}.
\bjtitle{Ann. Statist.}
\bvolume{20}(\bissue{4}),
\bfpage{1803}--\blpage{1827}
(\byear{1992})
\end{barticle}
\endbibitem

\bibitem[\protect\citeauthoryear{Diakonikolas and Kane}{2023}]{Diakonikolas2023HighDimRobustStats}
\begin{bbook}
\bauthor{\bsnm{Diakonikolas}, \binits{I.}},
\bauthor{\bsnm{Kane}, \binits{D.M.}}:
\bbtitle{Algorithmic High-Dimensional Robust Statistics}.
\bpublisher{Cambridge University Press},
\blocation{Cambridge}
(\byear{2023})
\end{bbook}
\endbibitem

\bibitem[\protect\citeauthoryear{Devroye et~al.}{2016}]{DevroyeGMoMNotSubGaus}
\begin{barticle}
\bauthor{\bsnm{Devroye}, \binits{L.}},
\bauthor{\bsnm{Lerasle}, \binits{M.}},
\bauthor{\bsnm{Lugosi}, \binits{G.}},
\bauthor{\bsnm{Oliveira}, \binits{R.I.}}:
\batitle{Sub-{G}aussian mean estimators}.
\bjtitle{Ann. Statist.}
\bvolume{44}(\bissue{6}),
\bfpage{2695}--\blpage{2725}
(\byear{2016})
\end{barticle}
\endbibitem

\bibitem[\protect\citeauthoryear{Eaton and Perlman}{1973}]{eaton:perlman:1973}
\begin{barticle}
\bauthor{\bsnm{Eaton}, \binits{M.L.}},
\bauthor{\bsnm{Perlman}, \binits{M.D.}}:
\batitle{The non-singularity of generalized sample covariance matrices}.
\bjtitle{Ann. Statist.}
\bvolume{1},
\bfpage{710}--\blpage{717}
(\byear{1973})
\end{barticle}
\endbibitem

\bibitem[\protect\citeauthoryear{Fawcett}{2006}]{FawcettIntroductionToROC}
\begin{barticle}
\bauthor{\bsnm{Fawcett}, \binits{T.}}:
\batitle{An introduction to {ROC} analysis}.
\bjtitle{Pattern Recognition Letters}
\bvolume{27}(\bissue{8}),
\bfpage{861}--\blpage{874}
(\byear{2006})
\end{barticle}
\endbibitem

\bibitem[\protect\citeauthoryear{Friedman et~al.}{2007}]{FriedmannCoordinateDescent}
\begin{barticle}
\bauthor{\bsnm{Friedman}, \binits{J.}},
\bauthor{\bsnm{Hastie}, \binits{T.}},
\bauthor{\bsnm{H\"{o}fling}, \binits{H.}},
\bauthor{\bsnm{Tibshirani}, \binits{R.}}:
\batitle{Pathwise coordinate optimization}.
\bjtitle{Ann. Appl. Stat.}
\bvolume{1}(\bissue{2}),
\bfpage{302}--\blpage{332}
(\byear{2007})
\end{barticle}
\endbibitem

\bibitem[\protect\citeauthoryear{Friedman et~al.}{2007}]{glasso}
\begin{barticle}
\bauthor{\bsnm{Friedman}, \binits{J.}},
\bauthor{\bsnm{Hastie}, \binits{T.}},
\bauthor{\bsnm{Tibshirani}, \binits{R.}}:
\batitle{Sparse inverse covariance estimation with the graphical lasso}.
\bjtitle{Biostatistics}
\bvolume{9}(\bissue{3}),
\bfpage{432}--\blpage{441}
(\byear{2007})
\end{barticle}
\endbibitem

\bibitem[\protect\citeauthoryear{Hyv\"{a}rinen}{2005}]{HyvaraninScoreM}
\begin{barticle}
\bauthor{\bsnm{Hyv\"{a}rinen}, \binits{A.}}:
\batitle{Estimation of non-normalized statistical models by score matching}.
\bjtitle{J. Mach. Learn. Res.}
\bvolume{6},
\bfpage{695}--\blpage{709}
(\byear{2005})
\end{barticle}
\endbibitem

\bibitem[\protect\citeauthoryear{Hyv\"{a}rinen}{2007}]{HyvarinenExtensions}
\begin{barticle}
\bauthor{\bsnm{Hyv\"{a}rinen}, \binits{A.}}:
\batitle{Some extensions of score matching}.
\bjtitle{Comput. Statist. Data Anal.}
\bvolume{51}(\bissue{5}),
\bfpage{2499}--\blpage{2512}
(\byear{2007})
\end{barticle}
\endbibitem

\bibitem[\protect\citeauthoryear{Inouye et~al.}{2016}]{InouyeSquareRootModel}
\begin{bchapter}
\bauthor{\bsnm{Inouye}, \binits{D.}},
\bauthor{\bsnm{Ravikumar}, \binits{P.}},
\bauthor{\bsnm{Dhillon}, \binits{I.}}:
\bctitle{Square root graphical models: Multivariate generalizations of univariate exponential families that permit positive dependencies}.
In: \bbtitle{Proceedings of The 33rd International Conference on Machine Learning}.
\bsertitle{Proceedings of Machine Learning Research},
vol. \bseriesno{48},
pp. \bfpage{2445}--\blpage{2453}.
\bpublisher{PMLR},
\blocation{New York, USA}
(\byear{2016})
\end{bchapter}
\endbibitem

\bibitem[\protect\citeauthoryear{Kuhn et~al.}{2024}]{kuhn2024distributionallyrobustoptimization}
\begin{botherref}
\oauthor{\bsnm{Kuhn}, \binits{D.}},
\oauthor{\bsnm{Shafiee}, \binits{S.}},
\oauthor{\bsnm{Wiesemann}, \binits{W.}}:
Distributionally Robust Optimization
(2024).
\url{https://arxiv.org/abs/2411.02549}
\end{botherref}
\endbibitem

\bibitem[\protect\citeauthoryear{Lauritzen}{1996}]{LauritzenGraphicalModels}
\begin{bbook}
\bauthor{\bsnm{Lauritzen}, \binits{S.L.}}:
\bbtitle{Graphical Models}.
\bsertitle{Oxford Statistical Science Series},
vol. \bseriesno{17},
p. \bfpage{298}.
\bpublisher{The Clarendon Press, Oxford University Press},
\blocation{New York, USA}
(\byear{1996})
\end{bbook}
\endbibitem

\bibitem[\protect\citeauthoryear{Lin et~al.}{2016}]{LinHighDimScoreMatching}
\begin{barticle}
\bauthor{\bsnm{Lin}, \binits{L.}},
\bauthor{\bsnm{Drton}, \binits{M.}},
\bauthor{\bsnm{Shojaie}, \binits{A.}}:
\batitle{Estimation of high-dimensional graphical models using regularized score matching}.
\bjtitle{Electron. J. Stat.}
\bvolume{10}(\bissue{1}),
\bfpage{806}--\blpage{854}
(\byear{2016})
\end{barticle}
\endbibitem

\bibitem[\protect\citeauthoryear{Lugosi and Mendelson}{2019}]{LugosiMeanEstimationHeavyTails}
\begin{barticle}
\bauthor{\bsnm{Lugosi}, \binits{G.}},
\bauthor{\bsnm{Mendelson}, \binits{S.}}:
\batitle{Mean estimation and regression under heavy-tailed distributions: A survey}.
\bjtitle{Foundations of Computational Mathematics}
\bvolume{19}(\bissue{5}),
\bfpage{1145}--\blpage{1190}
(\byear{2019})
\end{barticle}
\endbibitem

\bibitem[\protect\citeauthoryear{Liu et~al.}{2019}]{LiuTukeyMedComplexity}
\begin{barticle}
\bauthor{\bsnm{Liu}, \binits{X.}},
\bauthor{\bsnm{Mosler}, \binits{K.}},
\bauthor{\bsnm{Mozharovskyi}, \binits{P.}}:
\batitle{Fast computation of {T}ukey trimmed regions and median in dimension {$p>2$}}.
\bjtitle{J. Comput. Graph. Statist.}
\bvolume{28}(\bissue{3}),
\bfpage{682}--\blpage{697}
(\byear{2019})
\end{barticle}
\endbibitem

\bibitem[\protect\citeauthoryear{Lopuha\"{a} and Rousseeuw}{1991}]{LopuhaBreakdownPoints}
\begin{barticle}
\bauthor{\bsnm{Lopuha\"{a}}, \binits{H.P.}},
\bauthor{\bsnm{Rousseeuw}, \binits{P.J.}}:
\batitle{Breakdown points of affine equivariant estimators of multivariate location and covariance matrices}.
\bjtitle{Ann. Statist.}
\bvolume{19}(\bissue{1}),
\bfpage{229}--\blpage{248}
(\byear{1991})
\end{barticle}
\endbibitem

\bibitem[\protect\citeauthoryear{Laforgue et~al.}{2021}]{LaforgueMoMcorrupt}
\begin{bchapter}
\bauthor{\bsnm{Laforgue}, \binits{P.}},
\bauthor{\bsnm{Staerman}, \binits{G.}},
\bauthor{\bsnm{Cl{\'e}men{\c{c}}on}, \binits{S.}}:
\bctitle{Generalization bounds in the presence of outliers: a median-of-means study}.
In: \bbtitle{Proceedings of the 38th International Conference on Machine Learning}.
\bsertitle{Proceedings of Machine Learning Research},
vol. \bseriesno{139},
pp. \bfpage{5937}--\blpage{5947}.
\bpublisher{PMLR},
\blocation{New York, USA}
(\byear{2021})
\end{bchapter}
\endbibitem

\bibitem[\protect\citeauthoryear{Loh and Tan}{2018}]{Loh2018}
\begin{barticle}
\bauthor{\bsnm{Loh}, \binits{P.-L.}},
\bauthor{\bsnm{Tan}, \binits{X.L.}}:
\batitle{{High-dimensional robust precision matrix estimation: Cellwise corruption under $\epsilon $-contamination}}.
\bjtitle{Electronic Journal of Statistics}
\bvolume{12}(\bissue{1}),
\bfpage{1429}--\blpage{1467}
(\byear{2018})
\end{barticle}
\endbibitem

\bibitem[\protect\citeauthoryear{Maathuis et~al.}{2019}]{handbook}
\begin{bbook}
\beditor{\bsnm{Maathuis}, \binits{M.}},
\beditor{\bsnm{Drton}, \binits{M.}},
\beditor{\bsnm{Lauritzen}, \binits{S.}},
\beditor{\bsnm{Wainwright}, \binits{M.}} (eds.):
\bbtitle{Handbook of Graphical Models}.
\bsertitle{Chapman \& Hall/CRC Handbooks of Modern Statistical Methods},
p. \bfpage{536}.
\bpublisher{CRC Press},
\blocation{Boca Raton, FL}
(\byear{2019})
\end{bbook}
\endbibitem

\bibitem[\protect\citeauthoryear{Minsker}{2015}]{MinskerGeometricMoM}
\begin{barticle}
\bauthor{\bsnm{Minsker}, \binits{S.}}:
\batitle{Geometric median and robust estimation in {B}anach spaces}.
\bjtitle{Bernoulli}
\bvolume{21}(\bissue{4}),
\bfpage{2308}--\blpage{2335}
(\byear{2015})
\end{barticle}
\endbibitem

\bibitem[\protect\citeauthoryear{Minsker}{2019}]{MinskerMoMReview}
\begin{barticle}
\bauthor{\bsnm{Minsker}, \binits{S.}}:
\batitle{Distributed statistical estimation and rates of convergence in normal approximation}.
\bjtitle{Electron. J. Stat.}
\bvolume{13}(\bissue{2}),
\bfpage{5213}--\blpage{5252}
(\byear{2019})
\end{barticle}
\endbibitem

\bibitem[\protect\citeauthoryear{Maronna et~al.}{2019}]{MaronnaRobustStatistics}
\begin{bbook}
\bauthor{\bsnm{Maronna}, \binits{R.A.}},
\bauthor{\bsnm{Martin}, \binits{R.D.}},
\bauthor{\bsnm{Yohai}, \binits{V.J.}},
\bauthor{\bsnm{Salibi\'{a}n-Barrera}, \binits{M.}}:
\bbtitle{Robust Statistics},
\bedition{2}nd edn.
\bsertitle{Wiley Series in Probability and Statistics},
p. \bfpage{430}.
\bpublisher{John Wiley \& Sons, Inc.},
\blocation{Hoboken, NJ}
(\byear{2019})
\end{bbook}
\endbibitem

\bibitem[\protect\citeauthoryear{Niinimaa et~al.}{1990}]{NiinimaaOja0Bd}
\begin{barticle}
\bauthor{\bsnm{Niinimaa}, \binits{A.}},
\bauthor{\bsnm{Oja}, \binits{H.}},
\bauthor{\bsnm{Tableman}, \binits{M.}}:
\batitle{The finite-sample breakdown point of the {O}ja bivariate median and of the corresponding half-samples version}.
\bjtitle{Statist. Probab. Lett.}
\bvolume{10}(\bissue{4}),
\bfpage{325}--\blpage{328}
(\byear{1990})
\end{barticle}
\endbibitem

\bibitem[\protect\citeauthoryear{Oh and Deasy}{2014}]{OhSparseGeneRegulatoryNetworks}
\begin{barticle}
\bauthor{\bsnm{Oh}, \binits{J.H.}},
\bauthor{\bsnm{Deasy}, \binits{J.O.}}:
\batitle{Inference of radio-responsive gene regulatory networks using the graphical lasso algorithm}.
\bjtitle{BMC Bioinformatics}
\bvolume{15},
\bfpage{1}--\blpage{8}
(\byear{2014})
\end{barticle}
\endbibitem

\bibitem[\protect\citeauthoryear{Roy and Dunson}{2020}]{Roy2020CountGraphicalModelIntractableNormalizingConstant}
\begin{barticle}
\bauthor{\bsnm{Roy}, \binits{A.}},
\bauthor{\bsnm{Dunson}, \binits{D.B.}}:
\batitle{Nonparametric graphical model for counts}.
\bjtitle{Journal of Machine Learning Research}
\bvolume{21}(\bissue{229}),
\bfpage{1}--\blpage{21}
(\byear{2020})
\end{barticle}
\endbibitem

\bibitem[\protect\citeauthoryear{Ronkainen et~al.}{2003}]{RonkainenOjaMedComplexity}
\begin{bchapter}
\bauthor{\bsnm{Ronkainen}, \binits{T.}},
\bauthor{\bsnm{Oja}, \binits{H.}},
\bauthor{\bsnm{Orponen}, \binits{P.}}:
\bctitle{Computation of the multivariate {O}ja median}.
In: \bbtitle{Developments in Robust Statistics ({V}orau, 2001)},
pp. \bfpage{344}--\blpage{359}.
\bpublisher{Physica},
\blocation{Heidelberg}
(\byear{2003})
\end{bchapter}
\endbibitem

\bibitem[\protect\citeauthoryear{Serfling}{2006}]{serfling2006multivariate}
\begin{barticle}
\bauthor{\bsnm{Serfling}, \binits{R.J.}}:
\batitle{Multivariate symmetry and asymmetry}.
\bjtitle{Encyclopedia of Statistical Sciences}
\bvolume{8},
\bfpage{5338}--\blpage{5345}
(\byear{2006})
\end{barticle}
\endbibitem

\bibitem[\protect\citeauthoryear{Sun et~al.}{2015}]{SunRegScoreMatchingNeurips}
\begin{bchapter}
\bauthor{\bsnm{Sun}, \binits{S.}},
\bauthor{\bsnm{Kolar}, \binits{M.}},
\bauthor{\bsnm{Xu}, \binits{J.}}:
\bctitle{Learning structured densities via infinite dimensional exponential families}.
In: \bbtitle{Advances in Neural Information Processing Systems},
vol. \bseriesno{28}.
\bpublisher{Curran Associates, Inc.},
\blocation{New York, USA}
(\byear{2015})
\end{bchapter}
\endbibitem

\bibitem[\protect\citeauthoryear{Small}{1990}]{SmallMedianSurvey}
\begin{barticle}
\bauthor{\bsnm{Small}, \binits{C.G.}}:
\batitle{A survey of multidimensional medians}.
\bjtitle{International Statistical Review / Revue Internationale de Statistique}
\bvolume{58}(\bissue{3}),
\bfpage{263}--\blpage{277}
(\byear{1990})
\end{barticle}
\endbibitem

\bibitem[\protect\citeauthoryear{Tank et~al.}{2002}]{klein2002daily}
\begin{barticle}
\bauthor{\bsnm{Tank}, \binits{K.}},
\bauthor{\bsnm{A.M.G.}},
\bauthor{\bsnm{Coauthors}}:
\batitle{Daily dataset of 20th-century surface air temperature and precipitation series for the european climate assessment}.
\bjtitle{International Journal of Climatology}
\bvolume{22},
\bfpage{1441}--\blpage{1453}
(\byear{2002})
\end{barticle}
\endbibitem

\bibitem[\protect\citeauthoryear{Vardi and Zhang}{2001}]{VardiWeiszfeldImproved}
\begin{barticle}
\bauthor{\bsnm{Vardi}, \binits{Y.}},
\bauthor{\bsnm{Zhang}, \binits{C.-H.}}:
\batitle{A modified {W}eiszfeld algorithm for the {F}ermat-{W}eber location problem}.
\bjtitle{Math. Program.}
\bvolume{90}(\bissue{3}),
\bfpage{559}--\blpage{566}
(\byear{2001})
\end{barticle}
\endbibitem

\bibitem[\protect\citeauthoryear{Yu et~al.}{2019}]{YuScoreMatchNonNeg}
\begin{barticle}
\bauthor{\bsnm{Yu}, \binits{S.}},
\bauthor{\bsnm{Drton}, \binits{M.}},
\bauthor{\bsnm{Shojaie}, \binits{A.}}:
\batitle{Generalized score matching for non-negative data}.
\bjtitle{J. Mach. Learn. Res.}
\bvolume{20}(\bissue{1}),
\bfpage{2779}--\blpage{2848}
(\byear{2019})
\end{barticle}
\endbibitem

\bibitem[\protect\citeauthoryear{Yu et~al.}{2016}]{Mingyu2016statistical}
\begin{bchapter}
\bauthor{\bsnm{Yu}, \binits{M.}},
\bauthor{\bsnm{Kolar}, \binits{M.}},
\bauthor{\bsnm{Gupta}, \binits{V.}}:
\bctitle{Statistical inference for pairwise graphical models using score matching}.
In: \bbtitle{Advances in Neural Information Processing Systems},
vol. \bseriesno{29}.
\bpublisher{Curran Associates, Inc.},
\blocation{New York, USA}
(\byear{2016})
\end{bchapter}
\endbibitem

\end{thebibliography}

\appendix
\vspace{1cm}
\noindent{\LARGE\textbf{Appendix}}
\vspace{0.5cm}
\section{Regularity conditions to ensure the positive definiteness of \texorpdfstring{$\mathbf{\Gamma}$}{Gamma}}\label{suppSec:pd_conditions}

For the almost sure positive definiteness of the score matching design matrix $\mathbf{\Gamma}(x)\in\mathbb{R}^{r\times r}$, we require two assumptions on the sufficient statistic $\mathbf{t}\colon\mathbb{R}^m\to\mathbb{R}^r$: \begin{itemize}
    \item[(A1)] Without loss of generality, assume that $t_i$ is not constant for any $i\in \{1,\ldots, r\}$ 
    \item[(A2)] With $R_j := \{i\in \{1,\ldots, r\} \mid \partial_j t_i \neq 0\}$ and $d_j := |R_j|$ for $j\in\{1,\ldots, m\}$, define the function $\mathbf{v}^{(j)}\colon\mathbb{R}^m\to\mathbb{R}^{d_j}, \mathbf{v}^{(j)}(\xrandvect) := \sqrt{h_j(\xrandvect)}\cdot\partial_j \mathbf{t}|^{R_j}(\xrandvect)$. We assume that for any proper linear subspace $L$ of $\mathbb{R}^{d_j}$, the pre-image $\left(\mathbf{v}^{(j)} \right)^{-1}(L)$ is a Lebesgue null set in $\mathbb{R}^m$. 
\end{itemize}

\begin{example}
    In the square root graphical model with $\bs{\eta}$ known, it holds that $d_j=m$ and up to permutations of the components, $\mathbf{v}^{(j)}(\mathbf{x})=-\sqrt{h_j(x_j)/x_j}\cdot\sqrt{\mathbf{x}}$. If $h>0$ is invertible and sufficiently smooth, this is a diffeomorphism (its inverse equals $h^{-1}(y_j^2)( \mathbf{e}^{(j)} + (\mathbf{1} - \mathbf{e}^{(j)})(\mathbf{y}/y_j)^2)$ with the $j$-th Euclidean basis vector $\mathbf{e}^{(j)}$ and the all-one vector $\mathbf{1}$) and thus null sets, in particular proper linear subspaces, are mapped to null sets by the change of variables theorem for Lebesgue's measure.
\end{example}

\begin{lemma}\label{lemma:gamma_pos_def} 
Let $\xvect{1},\ldots, \xvect{n}\in\mathbb{R}^m$ be i.i.d. according to an exponential family satisfying (A1) and (A2). Further, let $(c_{ij})$ be variables on the same probability space that are positive almost surely. Assume $n\geq \max_{j=1,\ldots, m}\; d_j$. Then, $$\mathbf{M}:=\sum_{j=1}^m\sum_{i=1}^n c_{ij} h_j(\xvect{i}) \partial_j \mathbf{t}(\xvect{i}) \partial_j \mathbf{t}(\xvect{i})^\mathsf{T}$$ is positive definite almost surely.
\end{lemma}

\begin{proof}
Let $j\in \{1,\ldots, m\}$. We show that $(\mathbf{v}^{(j)}(\xvect{i}))_{i=1,\ldots, d_j}$ are independent almost surely. Note that this collection of vectors requires $n\geq d_j$. We show that the probability of linear dependence is zero: \begin{multline*}
    \pr{\mathbf{v}^{(j)}(\xvect{i})\in\linearspan(\mathbf{v}^{(j)}(\xvect{2}),\ldots, \mathbf{v}^{(j)}(\xvect{d_j}))} = \\
    \ex{\cprob{\xvect{i}\in (\mathbf{v}^{(j)})^{-1}\left(\linearspan(\mathbf{v}^{(j)}(\xvect2),\ldots, \mathbf{v}^{(j)}(\xvect{d_j}))\right)}{(\mathbf{v}^{(j)}(\xvect{i}))_{i=2,\ldots, d_j}}} \overset{(A2)}{=} 
    \ex{0} = 0.
\end{multline*}
Define the $d_j\times d_j$ matrices $\mathbf{M}^{(j)}(d) := \sum_{i=1}^{d} c_{ij} \mathbf{v}^{(j)}(\xvect{i})\mathbf{v}^{(j)}(\xvect{i})^\mathsf{T}$. The independence result implies that $\mathbf{M}^{(j)}(d_j)$ has full rank almost surely. Otherwise, there would be $\mathbf{v}\neq 0$ in its kernel by the rank theorem. This would imply $$\mathbf{M}^{(j)}(d_j) \mathbf{v} = \sum_{i=1}^{d_j} c_{ij} \mathbf{v}^{(j)}(\xvect{i})\mathbf{v}^{(j)}(\xvect{i})^\mathsf{T} \mathbf{v} = \sum_{i=1}^{d_j} \left(c_{ij} \mathbf{v}^{(j)}(\xvect{i})^\mathsf{T} \mathbf{v}\right) \mathbf{v}^{(j)}(\xvect{i}) = \mathbf{0},$$ a contradiction to $(\mathbf{v}^{(j)}(\xvect{i}))_{i=1,\ldots, d_j}$ being independent almost surely. Since $\mathbf{M}^{(j)}(d_j)$ is positive semidefinite due to the structure of its summands (recall $c_{ij}, h_j >0)$, it follows that $\mathbf{M}^{(j)}(d_j)$ is positive definite. Also, $\mathbf{M}^{(j)}(n)$ is positive definite since only more positive semidefinite terms are added.

To show the statement of this lemma, first note that $\mathbf{M}$ is positive semidefinite since $\sum_{i=1}^n c_{ij} h_j(\xvect{i}) \partial_j \mathbf{t}(\xvect{i}) \partial_j \mathbf{t}(\xvect{i})^\mathsf{T}\in\mathbb{R}^{r\times r}$ are positive semidefinite. Assume $\mathbf{M}$ had a nontrivial vector $\mathbf{v}\in\mathbb{R}^r$ in its kernel. By positive semidefiniteness of the summands, $\mathbf{M}\mathbf{v}=\mathbf{0}$ implies $\sum_{i=1}^n c_{ij} h_j(\xvect{i}) \partial_j \mathbf{t}(\xvect{i}) \partial_j \mathbf{t}(\xvect{i})^\mathsf{T}\mathbf{v} = 0$ for all $j$. Since $\partial_j t_i = 0$ for all $i\in \{1,\ldots, r\}\setminus R_j$ by definition of $R_j$, this is equivalent to $\mathbf{M}^{(j)}(n) \cdot \mathbf{v}_{R_j} = \mathbf{0}$. By the previous result on $\mathbf{M}^{(j)}(n)$, it follows that $\mathbf{v}_{R_j}=\mathbf{0}$. Assumption (A1) guarantees that $\{1,\ldots, r\}=\bigcup_{j=1,\ldots, m} R_j$, which implies $\mathbf{v}=\mathbf{0}$, a contradiction.
\end{proof}

A direct consequence of Lemma \ref{lemma:gamma_pos_def} is that $\overline{\mathbf{\Gamma}}(\mathbf{X})$ is positive definite almost surely (choose $c_i:= 1/n$). 

To see that the same holds for the GMoM version $\hat{\mathbf{\Gamma}}_K(\mathbf{X})$, consider the following equation from the paper, which holds when the geometric median does not equal one of its arguments: \begin{align*}
    \mathbf{m}:=\gmed{\hat{\bs{\mu}}^{(1)}, \ldots, \hat{\bs{\mu}}^{(K)}} = \frac{1}{\sum_{i=1}^K 1/\|\mathbf{m} - \hat{\bs{\mu}}^{(i)}\|_2}\sum_{i=1}^K \frac{\hat{\bs{\mu}}^{(i)}}{\|\mathbf{m} - \hat{\bs{\mu}}^{(i)}\|_2}.
\end{align*}
To apply this to $\hat{\mathbf{\Gamma}}_K(\mathbf{X})$, define \begin{align*}
    \hat{\bs{\mu}}^{(k)} := \frac{1}{n/K}\sum_{i=(k-1)K+1}^{kK}\sum_{j = 1}^m h_j(\xvect{i})\partial_j\mathbf{t}(\xvect{i}) \partial_j\mathbf{t}(\xvect{i})^\top, 
\end{align*} such that $\hat{\mathbf{\Gamma}}_K(\mathbf{X})=\gmed{\hat{\bs{\mu}}^{(1)}, \ldots, \hat{\bs{\mu}}^{(K)}}$. Lemma \ref{lemma:gamma_pos_def} guarantees positive definiteness of $\hat{\mathbf{\Gamma}}_K(\mathbf{X})$ with \begin{align*}
    c_i := \frac{K}{n}\left(\sum_{k=1}^K \frac{\|\hat{\mathbf{\Gamma}}_K(\mathbf{X}) - \hat{\bs{\mu}}^{(k_i)}\|_2}{\|\hat{\mathbf{\Gamma}}_K(\mathbf{X}) - \hat{\bs{\mu}}^{(k)}\|_2}\right)^{-1},
\end{align*} where $k_i$ is the block index that $i\in\{1,\ldots, n\}$ belongs to.

\section{Proof of Theorem 3.1}
{
\renewcommand{\thetheorem}{3.1}
\begin{theorem}
    Let $\xvect{1},\ldots,\xvect{n}\in\mathbb{R}^p$ be independent samples from a $p$-dimensional distribution with mean $\bs{\mu}$ and variance $\mathbf{\Sigma}$. Fix a confidence level of $0<\delta\leq 1$.  We allow for up to $n_c := (\lfloor 17\cdot\log(1/\delta)\rfloor + 1)\tau$ samples to be arbitrarily corrupted, where $0\leq\tau< 1/2$. Then, there exist functions $k(\tau)=\mathcal{O}(1/(\tfrac{1}{2} - \tau)^{2})$ and $c(\tau)=\mathcal{O}(1/(\tfrac{1}{2} - \tau)^{2.5})$ as $\tau\to\tfrac{1}{2}$ such that when the number of blocks $K$ defined as $K=K(\delta, \tau) :=\lfloor k(\tau)\cdot \log(1/\delta)\rfloor + 1$ satisfies $K\leq n/2$, it holds that  \begin{equation}\label{eqn_appendix:thm_gmom_concentration}
        \prob\biggl( \|\gmom{K}{\xvect{1},\ldots,\xvect{n}} - \bs{\mu}\|_2 >
        c(\tau)\sqrt{\log\left(\frac{4}{(1-\tau)^2}\frac{1}{\delta}\right)\frac{\tr(\mathbf{\Sigma})}{n}}\biggr) \leq \delta.
    \end{equation}
\end{theorem}
\addtocounter{theorem}{-1}
} 
This section proves the above theorem. Let \begin{align*}
    \psi(\alpha, p) &:= (1-\alpha)\log\left(\frac{1-\alpha}{1-p}\right) + \alpha\log\left(\frac{\alpha}{p}\right).
\end{align*}
We base the proof on the following robustness result on the geometric median of independent estimators from \citet[Remark 3.1.a]{MinskerGeometricMoM}. Set $C_\alpha:=(1-\alpha)/\sqrt{1 - 2\alpha}$ for $0 < \alpha < 1/2$.

\begin{lemma} [Minsker, 2015]\label{lemma:minsker_aggregation}
    Let $\bs{\mu}\in\mathbb{R}^p$,  and let $\hat{\bs{\mu}}_1,\ldots,\hat{\bs{\mu}}_k\in\mathbb{R}^p$ be a collection of independent estimators of $\bs{\mu}$. Let the hyperparameters $0<\alpha < 1/2$, $0<p<\alpha$ and $\varepsilon > 0$ be such that \begin{equation*} 
        \prob\left(\|\hat{\bs{\mu}}_j - \bs{\mu}\|_2 >\varepsilon\right)\leq p \qquad \forall j\in J,
    \end{equation*} where $J\subset\{1,\ldots,K\}$ has cardinality at least $(1-\tau) K$, and $\tau < \frac{\alpha - p}{1-p}$. Then   \begin{equation*}
        \prob\left(\|\Gmed(\hat{\bs{\mu}}_1,\ldots\hat{\bs{\mu}}_k) - \bs{\mu}\|_2 > C_\alpha\varepsilon\right)\leq e^{-K(1-\tau)\psi\left(\frac{\alpha - \tau}{1 - \tau}, p\right)}.
    \end{equation*}
\end{lemma}
The function $k(\tau)$ from the theorem statement can be set to $$k(\tau):=\frac{1}{(1-\tau)\psi\left(\frac{\left(1/2 - \tau\right)^2}{1-\tau}, \frac{1}{2}\left(\frac{1}{2} - \tau\right)^2\right)}$$ such that $K$ is given by
 $$K = K(\delta, \tau) := \lfloor k(\tau)\cdot \log(1/\delta)\rfloor + 1 = \left\lfloor\frac{\log(1/\delta)}{(1-\tau)\psi\left(\frac{\left(1/2 - \tau\right)^2}{1-\tau}, \frac{1}{2}\left(\frac{1}{2} - \tau\right)^2\right)}\right\rfloor + 1.$$
 The second function $c(\tau)$ from the theorem statement can be set to  $$c(\tau):=\frac{2\cdot(3/4 - \tau^2)}{(1/2 - \tau)\sqrt{1/2 - 2\tau^2}\sqrt{(1-\tau)\psi\left(\frac{\left(1/2 - \tau\right)^2}{1-\tau}, \frac{1}{2}\left(\frac{1}{2} - \tau\right)^2\right)}}.$$
It follows that $k(\tau)=\mathcal{O}(1/(\tfrac{1}{2} - \tau)^{2})$ and $c(\tau)=\mathcal{O}(1/(\tfrac{1}{2} - \tau)^{2.5})$ since $\log(1 - x)= \mathcal{O}(x)$ as $x\to 0$.

We can now prove the theorem for $k(\tau)$ and $c(\tau)$ given above.
\begin{proof}
    To simplify notation, let  $$\hat{\bs{\mu}}:= \gmom{K}{\xvect{1},\ldots,\xvect{K\cdot\lfloor n/K\rfloor}} = \Gmed(\hat{\bs{\mu}}_1,\ldots,\hat{\bs{\mu}}_K).$$ For theoretical simplicity, we prove the theorem for $\hat{\bs{\mu}}$ with $K$ blocks of equal block size $\lfloor n/K\rfloor$ 

    The main step of this proof is applying Lemma \ref{lemma:minsker_aggregation} to the block means $\hat{\bs{\mu}}_1,\ldots,\hat{\bs{\mu}}_K$. We start by fixing $\alpha, p$ and $\varepsilon$ in the Lemma. Consider the following choices that depend on the corruption parameter $\tau$: \begin{align*}
        p(\tau):=&\frac{1}{2}\left(\frac{1}{2} - \tau\right)^2,\\
        \alpha(\tau):=&2p(\tau) + \tau = \tau^2 + \frac{1}{4},\\
        \varepsilon(\tau):=&\sqrt{\frac{2K\tr(\mathbf{\Sigma})}{n\,p(\tau)}}.
    \end{align*}
    It remains to verify that these choices can satisfy the conditions in Lemma \ref{lemma:minsker_aggregation}. To choose the set $J$, first note that $K(\delta,\cdot)$ is an increasing function. 
   By assumption, at most $\tau K(\delta,0)$ samples are corrupted (since $17 \leq 1/\psi(1/4, 1/8)$). So, the proportion of corrupted blocks is at most $$(\tau K(\delta,0))/K(\delta,\tau) = \tau (K(\delta,0)/K(\delta,\tau))\leq \tau \cdot 1 = \tau.$$ Therefore, we can set $J$ to be the set of uncorrupted blocks.

    To show the probabilistic bound for all blocks $j\in J$, we assume w.l.o.g. that $j=1$. Using the fact that $\lfloor n/K\rfloor^{-1}\leq 2K/n$ due to $K\leq n/2$, we find \begin{multline*}\ex{\|\hat{\bs{\mu}}_1 - \bs{\mu}\|_2^2} = \frac{1}{\lfloor n/K\rfloor ^2}\sum_{i,j=1}^{\lfloor n/K\rfloor}\ex{(\xvect{i} - \bs{\mu})^T(\xvect{j} - \bs{\mu})} = \\
    \frac{1}{\lfloor n/K\rfloor ^2} \sum_{i=1}^{\lfloor n/K\rfloor}\ex{(\xvect{i} - \bs{\mu})^T(\xvect{i} - \bs{\mu})} = \frac{\ex{\|\xrandvect - \bs{\mu}\|_2^2}}{\lfloor n/K\rfloor}\leq \frac{2K}{n}\tr(\mathbf{\Sigma}).\end{multline*} The probabilistic bound now follows from Chebycheff's inequality, where everything but $p(\tau)$ cancels.

    For the second condition, check that $$\frac{\alpha(\tau) - p(\tau)}{1-p(\tau)}=\frac{2\,p(\tau) + \tau - p(\tau)}{1-p(\tau)} = \frac{p(\tau)}{1-p(\tau)} + \frac{\tau}{1-p(\tau)}> 0 + \frac{\tau}{1}=\tau.$$

    By Lemma \ref{lemma:minsker_aggregation}, we have established \begin{equation}\label{eq:gmom_concentration_proof_v3_1}
        \pr{\|\hat{\bs{\mu}} - \bs{\mu} \|_2 > C_{\alpha(\tau)}\varepsilon(\tau)} \overset{\ref{lemma:minsker_aggregation}}{\leq}
        e^{-K(1-\tau)\psi\left(\frac{\alpha(\tau) - \tau}{1 - \tau}, p(\tau)\right)}.
    \end{equation}

    We start by simplifying the exponent in the right hand side of (\ref{eq:gmom_concentration_proof_v3_1}) for our choice of $K, \alpha(\tau)$ and $p(\tau)$.  We drop the dependency on $\tau$ for simplicity. First, note that $$K=\left\lfloor\frac{\log(1/\delta)}{(1-\tau)\psi\left(\frac{2p}{1-\tau},p\right)}\right\rfloor + 1,$$ which allows the following simplifications: \begin{multline*}
        K(1-\tau)\psi\left(\frac{\alpha - \tau}{1 - \tau}, p\right) = \left(\left\lfloor\frac{\log(1/\delta)}{(1-\tau)\psi\left(\frac{2p}{1-\tau},p\right)}\right\rfloor + 1\right)(1-\tau)\psi\left(\frac{2p}{1 - \tau}, p\right) \\
        \overset{\text{for some }c\in[0,1)}{=}\left(\frac{\log(1/\delta)}{(1-\tau)\psi\left(\frac{2p}{1-\tau},p\right)} - c + 1\right)(1-\tau)\psi\left(\frac{2p}{1 - \tau}, p\right) = \\
        \log(1/\delta) + (1-c)(1-\tau)\psi\left(\frac{2p}{1 - \tau}, p\right) \overset{c \in[0,1)}{\geq} \log(1/\delta) + 0 = \log(1/\delta).
    \end{multline*}
    Since the negative of the initial term is the exponent, we can bound the right hand side of (\ref{eq:gmom_concentration_proof_v3_1}) by \begin{equation*}
    e^{-K(1-\tau)\psi\left(\frac{\alpha - \tau}{1 - \tau}, p\right)}\leq e^{-\log(1/\delta)}=e^{\log(\delta)}=\delta.\end{equation*}
    All that remains is to simplify $C_\alpha\varepsilon$: \begin{multline*}
        C_\alpha\varepsilon = C_\alpha\sqrt{\frac{2K\tr(\mathbf{\Sigma})}{np}} = \frac{C_\alpha\sqrt{2}}{\sqrt{p}\sqrt{(1-\tau)\psi\left(\frac{\alpha - \tau}{1 - \tau}, p\right)}}\cdot
        \sqrt{ K\cdot (1-\tau)\psi\left(\frac{\alpha - \tau}{1 - \tau}, p\right)} \cdot\sqrt{\frac{\tr(\mathbf{\Sigma})}{n}}\leq\\
        c(\tau)\sqrt{\left(\frac{\log(1/\delta)}{(1-\tau)\psi\left(\frac{2p}{1-\tau},p\right)} + 1\right)\cdot (1-\tau)\psi\left(\frac{2p}{1 - \tau}, p\right)}\cdot \sqrt{\frac{\tr(\mathbf{\Sigma})}{n}} = \\
        c(\tau)\sqrt{\log(1/\delta) + (1-\tau)\psi\left(\frac{2p}{1 - \tau}, p\right)}\cdot\sqrt{\frac{\tr(\mathbf{\Sigma})}{n}} \overset{\text{First term of }\psi\text{ negative}}{\leq}\\
        c(\tau)\sqrt{\log(1/\delta) + (1-\tau)\frac{2p}{1- \tau}\log\left(\frac{2p}{(1-\tau)p}\right)}\cdot\sqrt{\frac{\tr(\mathbf{\Sigma})}{n}} \overset{p\leq 1}{\leq} \\
        c(\tau)\sqrt{\log(1/\delta) + 2\log\left(\frac{2}{1-\tau}\right)}\cdot\sqrt{\frac{\tr(\mathbf{\Sigma})}{n}} = c(\tau)\sqrt{\log\left(\frac{4}{(1-\tau)^2\cdot \delta}\right)}\cdot\sqrt{\frac{\tr(\mathbf{\Sigma})}{n}}.
    \end{multline*}
\end{proof}

\section{Proof of Corollary 3.2}

{
\renewcommand{\thetheorem}{3.2}
\begin{corollary}
    Let $\xvect{1},\ldots,\xvect{n}\in\mathbb{R}^p$ be independent samples from an exponential family with parameter $\bs{\theta}_0\in\mathbb{R}^r$. Assume that the following regularity conditions are met: \begin{itemize}
        \item[1)] The conditions from Proposition 2 in \cite{YuScoreMatchNonNeg} hold, i.e. the exponential family satisfies the basic requirements for score matching and, if its support is restricted, the dampening function $\mathbf{h}$ satisfies regularity assumptions. 
        \item[2)] The exponential family satisfies (A1) and (A2) from Section \ref{suppSec:pd_conditions} of this supplement.
        \item[3)] $\mathbf{\Gamma}(\xrandvect)$ and $\mathbf{g}(\xrandvect)$ have finite second moments when $\xrandvect$ comes from the distribution indexed by $\bs{\theta}_0$
    \end{itemize}
    Allow for up to $n_c$ samples to be arbitrarily corrupted, where $n_c=o(n)$. Then, there exists a sequence $K=K(n_c)$ such that the robust score matching estimator $\hat{\bs{\theta}}(K(n_c))$ converges against $\bs{\theta}_0$ in probability when $n\to\infty$. 
\end{corollary}
\addtocounter{theorem}{-1}
}

\begin{proof}
    Assume without loss of generality that $n_c\to\infty$ as $n\to\infty$. Should the true number of corrupt samples be bounded by $M<\infty$, it certainly holds true that at most $n_c := M + \log(n)$ samples have been corrupted. Hence, the theorem assumptions are also satisfied for this larger $n_c$, which diverges as $n\to\infty$. 

    Fix some $0 < \tau_0 < 1/2$ independent of $n$. Setting $\delta(n_c) := \exp(- (\lceil n_c/\tau_0\rceil - 1)/17)$, it holds that \begin{equation*}
        (\lfloor 17\cdot\log(1/\delta(n_c))\rfloor + 1)\tau_0 = (\lfloor \lceil n_c/\tau_0\rceil - 1\rfloor + 1) \tau_0 = \lceil n_c/\tau_0\rceil \tau_0 \geq n_c,
    \end{equation*}so in other words $(\delta(n_c), n_c, \tau_0)$ satisfy the assumption of Theorem 3.1. We set $K(n_c) := \lfloor k(\tau_0)\log(1/\delta(n_c))\rfloor + 1$ in line with Theorem 3.1. By our choice of $\delta(n_c)$, we have $K(n_c) = \lfloor \frac{k(\tau_0)}{17} (\lceil n_c/\tau_0\rceil - 1)\rfloor + 1 = o(n)$ by the assumption $n_c=o(n)$. Consequently, $K(n_c)\leq n/2$ for $n$ large enough, which is the last assumption of Theorem 3.1 to check.

    Denoting by $\mathbf{X}$ the $n\times p$ matrix having $\xvect{i}$ as the i-th row and \begin{equation*}B(n, \delta, \mathbf{\Sigma}) := c(\tau_0)\sqrt{\log\left(\frac{4}{(1-\tau_0)^2}\frac{1}{\delta}\right)\frac{\tr(\mathbf{\Sigma})}{n}},\end{equation*} Theorem 3.1 guarantees that \begin{align}\label{eq:cor32_p1}
        \pr{\|\hat{\mathbf{\Gamma}}_{K(n_c)}(\mathbf{X}) - \mathbf{\Gamma}_0\|> B(n, \delta(n_c), \Sigma_\mathbf{\Gamma})} &\leq\delta(n_c), \\
        \pr{\|\hat{\mathbf{g}}_{K(n_c)}(\mathbf{X}) - \mathbf{g}_0\|> B(n, \delta(n_c), \Sigma_\mathbf{g})} &\leq\delta(n_c),
    \end{align} where $\mathbf{\Sigma}_\mathbf{\Gamma}$ denotes the variance of $\mathbf{\Gamma}(\xrandvect)$ and $\mathbf{\Sigma}_\mathbf{g}$ that of $\mathbf{g}(\xrandvect)$ when $\xrandvect$ is distributed according to $\bs{\theta}_0$.

    Since  $\log(1/\delta(n_c)) = (\lceil n_c/\tau_0\rceil - 1)/17 = o(n)$, we have that $B(n, \delta(n_c), \Sigma_\mathbf{\Gamma})$ and $B(n, \delta(n_c), \Sigma_\mathbf{g})$ converge to zero as $n\to\infty$. Also, since we assumed without loss of generality that $n_c\to\infty$ as $n\to\infty$, we have that $\delta(n_c)\to 0$ when $n\to\infty$. These observations together with \eqref{eq:cor32_p1} imply that \begin{align}
        \hat{\mathbf{\Gamma}}_{K(n_c)}(\mathbf{X}) &\xrightarrow{P} \mathbf{\Gamma}_0, & \hat{\mathbf{g}}_{K(n_c)}(\mathbf{X}) &\xrightarrow{P} \mathbf{g}_0
    \end{align} in probability as $n\to\infty$. As matrix inversion and multiplication are continuous, we have that \begin{multline*}
         \hat{\bs{\theta}}(K(n_c)) := \underset{\bs{\theta}\in\Omega}{\argmin}\; \tfrac{1}{2}\bs{\theta}^\top \hat{\mathbf{\Gamma}}_{K(n_c)}(\mathbf{X}) \bs{\theta} - \bs{\theta}^\top \hat{\mathbf{g}}_{K(n_c)}(\mathbf{X}) \overset{(*1)}{=} \\
         \hat{\mathbf{\Gamma}}_{K(n_c)}(\mathbf{X})^{-1} \hat{\mathbf{g}}_{K(n_c)}(\mathbf{X}) \quad\xrightarrow{P}\quad \mathbf{\Gamma}_0^{-1}\mathbf{g}_0 \overset{(*2)}{=} \bs{\theta}_0,
    \end{multline*} where the minimizer in equation $(*1)$ exists almost surely and is given by the matrix inversion formula because of assumption 2) in the corollary  statement above, and equation $(*2)$ holds because of assumption 1) in the corollary statement.
\end{proof}

\section{Choice of \texorpdfstring{$K$}{K}}\label{appendix:choice_k}
\begin{figure}
    \centering
    \includegraphics[scale=0.5]{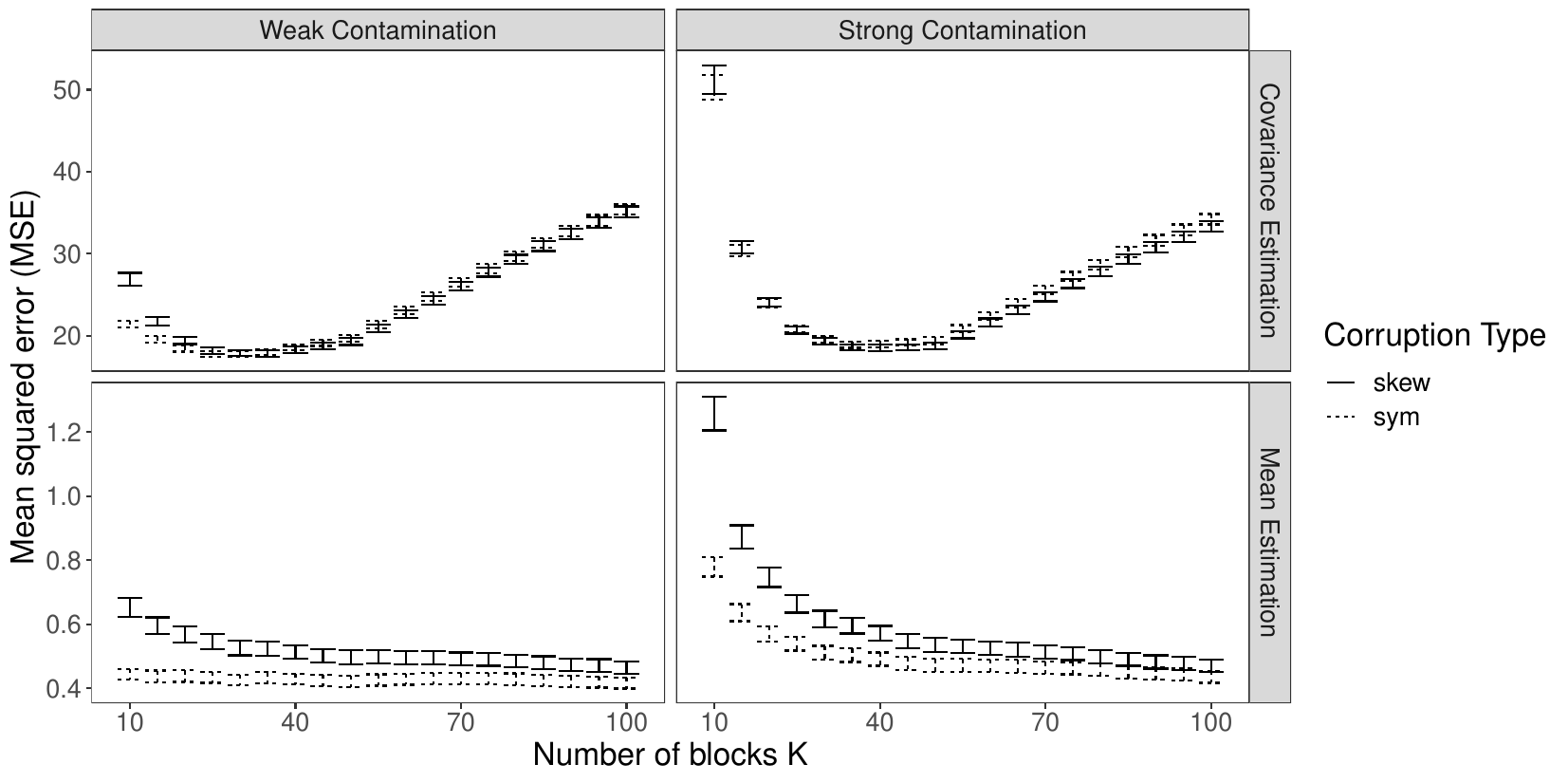}
    \caption{MSE versus number of blocks $K$ for Gaussian covariance (top) and mean (bottom) estimation. $5\%$ of observations were corrupted at varying intensity (weak (left) versus strong (right)) and using different types of corrupting distributions (skewed versus symmetric). }
    \label{fig:K_simulation_corruption}
\end{figure}

To understand what choice for the number of blocks $K$ in the GMoM leads to the best mean squared error (MSE) under contamination, two simulation studies are conducted. In the first, we estimate a Gaussian mean vector and a covariance matrix from contaminated data. Conceptually, this simulation corresponds to estimating $\bs{\Gamma}$ and $\bs{g}$ individually. Figure \ref{fig:K_simulation_corruption} displays the results. In view of the simulation results, we propose $K:= 4\varepsilon n$ as a heuristic.

In the second simulation, we use the robust score matching estimator $\hat{\bs{\theta}}(K)$ from section 3.2 of the main paper to estimate the parameters of a square root graphical model from contaminated data. This simulation sheds some light on how to tune the number of blocks under contamination if one cares about the downstream accuracy of a score matching estimator that incorporates $\bs{\Gamma}$ and $\bs{g}$ estimated through a GMoM procedure. Figure \ref{fig:K_choice_sqr} validates the heuristic $K:= 4\varepsilon n$ proposed earlier.

\subsection{Gaussian mean and covariance estimation}\label{subsec:gaussian_mean_cov}

\textbf{Two estimation problems} are considered, involving a $10$-dimensional Gaussian distribution with mean $\bs{\mu}=0$ and randomly fixed covariance matrix $\mathbf{\Sigma}$. The first problem is finding the population mean $\bs{\mu}$, which is a classic problem of general interest. From a sample $\xvect{1}, \ldots\xvect{n}$ with $n=100$, the mean is estimated simply by \begin{align*}
    \hat{\bs{\mu}}:=\gmom{K}{\xvect{1}, \ldots, \xvect{n}}.
\end{align*} The second problem is estimating the Gaussian covariance matrix $\mathbf{\Sigma}$, a problem that is structurally similar to estimating $\Gammanaught$ in score matching, especially for pairwise interaction models. From a sample $\xvect{1}, \ldots\xvect{100}$ from $N(\mathbf{0}, \mathbf{\Sigma})$, the covariance is estimated by \begin{align*}
    \hat{\mathbf{\Sigma}}:=\gmom{K}{\xvect{1}\cdot{\xvect{1}}^\top,\ldots, \xvect{n}\cdot{\xvect{n}}^\top}.
\end{align*} Both problems together cover a range of distributional properties. When estimating a Gaussian mean, the underlying distribution is symmetric and has light tails. Conversely, when estimating the Gaussian covariance, the underlying Wishart distribution is not symmetric and has heavier tails than the Gaussian. 

\textbf{Four contamination scenarios} were considered, combining two levels of corruption intensity with two types of corrupting distributions. Each scenario corrupted $5$ observations $\xvect{i}$ of the sample $\xvect{1}, \ldots, \xvect{100}$ with independent draws from a corrupting distribution. On the corrupted sample, $\hat{\bs{\mu}}$ and $\hat{\mathbf{\Sigma}}$ were computed. The corruption intensity was varied by setting $\alpha$ to $2$ or $10$ in the following distributions. The first corrupting distribution was Gaussian with mean zero and covariance $\alpha\cdot\diag(\hat{\sigma}_1^2,\ldots,\hat{\sigma}_{10}^2)$, where $\hat{\sigma}_i$ was estimated from $\xvect{1}, \ldots, \xvect{100}$. This is labeled as \emph{sym} in Figure \ref{fig:K_simulation_corruption}. The second corrupting distribution was a Pareto with independent components $(P_1,\ldots, P_{10})$. Each Pareto $P_i$ had its scale parameter chosen such that its $3/4$-th quantile equalled $\alpha$ times the $3/4$-th quantile of $\xvect{1}_i, \ldots, \xvect{100}_i$, and the distribution was shifted such that the lower Pareto cutoff agreed with the population mean of $0$. This is labeled as \emph{skew} in Figure \ref{fig:K_simulation_corruption}.

\textbf{Each simulation output} is comprised of the empirical error $\|\hat{\bs{\mu}} - \mathbf{0}\|_2^2$ and $\|\hat{\mathbf{\Sigma}} - \mathbf{\Sigma}\|_2^2$ for the two estimation problems and four contamination scenarios, respectively, where the number of blocks $K$ was ranged from the breakdown point of $K=10$ to the geometric median $K=100$. Aggregates over $1000$ simulations are reported in Figure \ref{fig:K_simulation_corruption}.  

\textbf{Interpretation:} Increased corruption strength degrades the MSE for $K\lesssim 40$, but has little effect on MSE for large $K$. It seems that once a comfortable distance from the breakdown point is reached, the corruption is irrelevant. The MSE curves share similar shapes per corruption type, showing that the GMoM reacts similarly to different corrupting distributions. The number of blocks $K$ resulting in optimal MSE is $K=n=100$ for mean estimation and $K\approx 30$ for covariance estimation. The different optimum can be explained by the fact that $\hat{\mathbf{\Sigma}}$ with $K=n$ is biased for $\mathbf{\Sigma}$, which degrades the MSE for large $K$ in covariance estimation, while $\hat{\bs{\mu}}$ is unbiased.

\textbf{Take away:} A choice for $K$ optimizing the MSE of the GmoM when a proportion $\varepsilon$ of samples is contaminated should get some distance to the breakdown point of $2\varepsilon n$, however not be to large in order to avoid the bias witnessed in covariance estimation. We propose $K:=4\varepsilon n$. 

\subsection{Estimating a square root graphical model with score matching}
\begin{figure}
     \centering
     \begin{subfigure}{0.45\textwidth}
         \centering
         \includegraphics[scale=0.8]{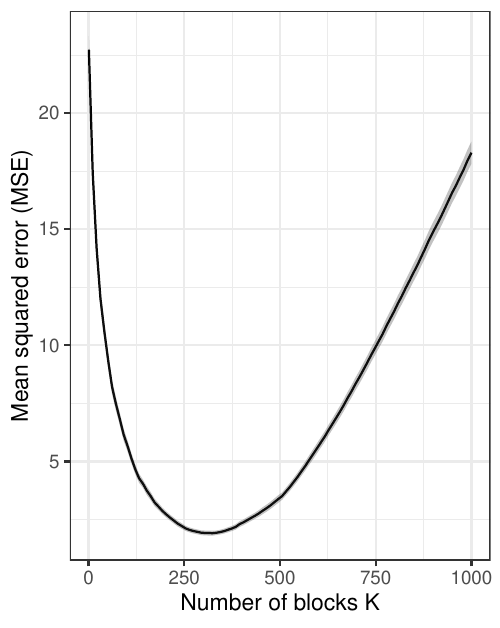}
     \end{subfigure}
     \begin{subfigure}{0.45\textwidth}
         \centering
         \includegraphics[scale=0.8]{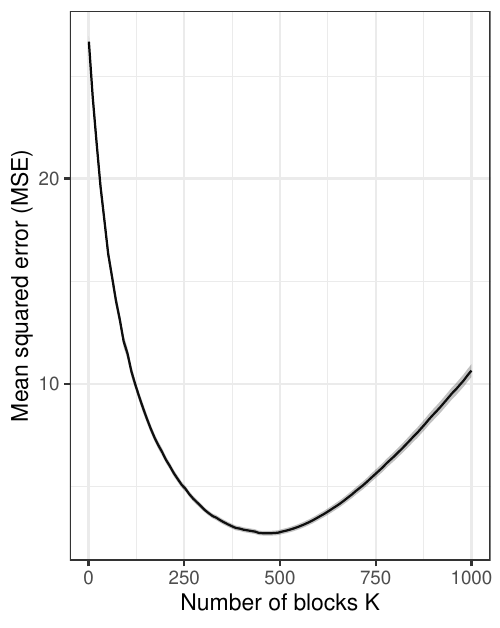}
     \end{subfigure}
        \caption{MSE versus number of blocks $K$ in the robust score matching procedure $\hat{\bs{\theta}}(K)$ for estimating the parameters of a square root graphical model under $5$\% contamination (left) and $10$\% contamination (right).}
        \label{fig:K_choice_sqr}
\end{figure}

\textbf{Data generation:} As an example of an exponential family of interest, we consider the square root graphical model introduced in the main paper. We consider a $m=5$ dimensional model with an interaction matrix $\bs{\Theta}$ and coefficient vector $\bs{\eta}$ being randomly determined. From this model, $n=1000$ samples were created. Of these, $5$\% and $10$\% respectively were contaminated by draws from a Pareto distribution, with parameters chosen such that most contaminated data points remained close to the mean of the uncorrupted data, while some became serious outliers. 

\textbf{Estimation:} Estimates $\hat{\bs{\Theta}}$ and $\hat{\bs{\eta}}$ were obtained using the robust score matching estimator $\hat{\bs{\theta}}(K)$ from section 3.2 of the main paper, where the number of blocks $K$ was varied on a grid from $1$ to $n$, corresponding to the mean and geometric median on the extremes of the spectrum.

\textbf{Simulation target:} For each contaminated data set and choice of $K$, the squared error $\text{SE}:= \|\hat{\bs{\Theta}} - \bs{\Theta}\|_2^2 + \|\hat{\bs{\eta}} - \bs{\eta}\|_2^2$ was computed. Averages and uncertainty estimates over $N=100$ independent Monte Carlo repetitions are presented in Figure \ref{fig:K_choice_sqr}.

\textbf{Take away:} In Figure \ref{fig:K_choice_sqr}, the number of blocks $K$ minimizing the MSE increases as the number of contaminated samples increases (left ($5$\% contamination): $K_\text{opt}\approx 320$, right ($10$\% contamination): $K_\text{opt}\approx 460$). This agrees with the intuition that a higher number of blocks makes the GMoM tolerate a higher number of outliers. The proposed choice $K:= 4 \varepsilon n$ corresponds to $K=200$ and $K=400$ respectively. Both fall in area of low MSE under corruption respectively and are thus suitable choices of the block-size. Still, they don't quite optimize the MSE, highlighting the need for further research into how to optimally tune the geometric median of means under contamination. 

Also note that compared to the simulations in Figure \ref{fig:K_simulation_corruption}, the MSE curves of the score matching estimator resemble that of Gaussian covariance estimation more closely than that of Gaussian mean estimation. This is not surprising, given that the distributions of $\bs{\Gamma}$ and $\bs{g}$ need not be centrally symmetric.

\section{Proof of Theorem 4.3}
We begin with a lemma that extends Theorem 3.1 by allowing for a diagonal multiplier. Set $\mathbf{b}:= \beta\cdot\vect{\mathbf{I}_m}$ with the vectorized $m\times m$ identity matrix $\mathbf{I}_m$ to obtain the diagonal multiplier $\beta$ as it is used in the paper.
\begin{lemma}\label{lemma:gmom_conc_diag_mult}
    Let $\xvect{1},\ldots,\xvect{n}\in\mathbb{R}^p$ be independent samples from a $p$-dimensional distribution with mean $\bs{\mu}$ and covariance $\mathbf{\Sigma}$. Fix a confidence level $\delta\in(0,1]$.  We allow for up to $\tau (\lfloor 17\log(1/\delta)\rfloor + 1)$ samples to be arbitrarily corrupted, where $0\leq\tau< 1/2$. Split the samples into $K$ blocks of equal size $\lfloor\frac{n}{K}\rfloor$, where $K = K(\delta, \tau)$ as in Theorem 3.1. Further, let $c(\tau)$ as in Theorem 3.1.
    
    Assume that $\tr(\mathbf{\Sigma}) >0$, and let $\mathbf{b}\in\mathbb{R}^p$ such that $$\|\mathbf{b}\|_\infty \leq \frac{1}{1 + (\|\bs{\mu}\|_2/\sqrt{2\tr(\mathbf{\Sigma})})\,\sqrt{n/K}}.$$
    
    If for the confidence level $\delta$ it holds that $K\leq n/2$, then \begin{equation*}
        \prob\biggl( \|\left(1 + \mathbf{b}\right)\circ \gmom{K}{\xvect{1},\ldots,\xvect{n}} - \ex{X}\|_\infty > 
        2\cdot c(\tau)\sqrt{\log\left(\frac{4}{(1-\tau)^2}\frac{1}{\delta}\right)\frac{\tr(\mathbf{\Sigma})}{n}}\biggr) \leq \delta,
    \end{equation*}
    where $\circ$ denotes elementwise multiplication.
\end{lemma}
\begin{proof}
    To simplify notation, let \begin{align*}
        \hat{\bs{\mu}}& :=\gmom{K}{\xvect{1},\ldots,\xvect{n}}, & t &:= c(\tau)\sqrt{\log\left(\frac{4}{(1-\tau)^2}\frac{1}{\delta}\right)\frac{\tr(\mathbf{\Sigma})}{n}}.
    \end{align*}
     We show the implication \begin{equation}\label{eq:estimator:perf:gmom_plus_diag_concentration_proof_1}
         \|\hat{\bs{\mu}} - \bs{\mu}\|_2 \leq t \implies \| \mathbf{b}\circ \hat{\bs{\mu}}\|_2 \leq t.
     \end{equation}   
    If the left hand side of \eqref{eq:estimator:perf:gmom_plus_diag_concentration_proof_1} holds, we find (recall $t >0$ since $\tr(\mathbf{\Sigma}) > 0 $ )
     \begin{equation}\label{eq:estimator:perf:gmom_plus_diag_concentration_proof_2}
        \|\hat{\bs{\mu}}\|_2 \leq \|\bs{\mu}\|_2 + t = t \left( \frac{\|\bs{\mu}\|_2}{t} + 1\right) \iff \frac{1}{1 + \|\bs{\mu}\|_2 / t}\|\hat{\bs{\mu}}\|_2 \leq t.
    \end{equation}
    We can use \eqref{eq:estimator:perf:gmom_plus_diag_concentration_proof_2} for the right hand side of \eqref{eq:estimator:perf:gmom_plus_diag_concentration_proof_1}. Recalling the definitions of $\alpha(\tau), p(\tau)$ and $\varepsilon(\tau)$ from the proof of Theorem~3.1 as well as the fact that the end of said proof can be rephrased as $C_{\alpha(\tau)}\varepsilon(\tau)\leq t$, we find \begin{multline*}
        \| \mathbf{b}\circ \hat{\bs{\mu}}\|_2 \leq \frac{1}{1 + (\|\bs{\mu}\|_2/\sqrt{2\tr(\mathbf{\Sigma})})\,\sqrt{n/K}}\|\hat{\bs{\mu}}\|_2 \overset{\sqrt{p(\tau)}\leq 1}{\leq} \\
        \frac{1}{1 + \sqrt{p(\tau)}(\|\bs{\mu}\|_2/\sqrt{2\tr(\mathbf{\Sigma})})\,\sqrt{n/K}}\|\hat{\bs{\mu}}\|_2 \overset{C_{\alpha(\tau)}\geq 1}{\leq} \\
        \frac{1}{1 + \|\bs{\mu}\|_2 / (C_{\alpha(\tau)} \varepsilon(\tau))}\|\hat{\bs{\mu}}\|_2 \leq \frac{1}{1 + \|\bs{\mu}\|_2 / t}\|\hat{\bs{\mu}}\|_2\overset{\eqref{eq:estimator:perf:gmom_plus_diag_concentration_proof_2}}{\leq} t.
    \end{multline*}
    This bound proves \eqref{eq:estimator:perf:gmom_plus_diag_concentration_proof_1}
    which allows us to deduce the inclusion of events
    \begin{equation*}
        \{\|\hat{\bs{\mu}} - \bs{\mu} + \mathbf{b}\circ\hat{\bs{\mu}}\|_2 > 2t\} \subset \{\|\hat{\bs{\mu}} - \bs{\mu} \|_2 + \|\mathbf{b}\circ\hat{\bs{\mu}}\|_2 > 2t\} \overset{\eqref{eq:estimator:perf:gmom_plus_diag_concentration_proof_1}}{\subset} 
        \{\|\hat{\bs{\mu}} - \bs{\mu} \|_2 > t\}.
    \end{equation*}
    Hence, by inclusion of events $\{\|\cdot\|_\infty\geq 2t\}\subset\{\|\cdot\|_2\geq 2t\}$ and Theorem 3.1 $$\pr{\|(1 + \mathbf{b})\circ\hat{\bs{\mu}} - \bs{\mu}\|_\infty > 2t}\leq \pr{\|\hat{\bs{\mu}} - \bs{\mu} \|_2 > t} \overset{\text{Thm 3.1}}{\leq} \delta.$$
\end{proof}

Our strategy is now to apply the following theorem from \citet{YuScoreMatchNonNeg}:

\begin{theorem}[Yu et al.]\label{thm:score_m_success_guarantee}
Suppose, $\mathbf{\Gamma}_{0, S_0S_0}$ is invertible and satisfies the irrepresentability condition with incoherence parameter $\alpha$. Assume $$\|\left(\hat{\mathbf{\Gamma}}_K + \beta\cdot\diag(\hat{\mathbf{\Gamma}}_K)\right) - \mathbf{\Gamma}_0\|_\infty<\varepsilon_1, \qquad \|\hat{\mathbf{g}}_K - \mathbf{g}_0\|_\infty<\varepsilon_2,$$ and $d_{\bs{\theta}_0}\varepsilon_1\leq\alpha/(6c_{\mathbf{\Gamma}_0})$. If $$\lambda > \frac{3(2 - \alpha)}{\alpha}\max(c_{\bs{\theta}_0}\varepsilon_1, \varepsilon_2),$$ then it holds that he minimizer $\hat{\bs{\theta}}(K, \beta, \lambda)$ is unique with $S(\hat{\bs{\theta}}(K, \beta, \lambda)) \subset S_0$ and satisfies $$\|\hat{\bs{\theta}}(K, \beta, \lambda) - \bs{\theta}_0\|_\infty\leq\frac{c_{\mathbf{\Gamma}_0}}{2 - \alpha}\lambda.$$
\end{theorem}

Combining Lemma \ref{lemma:gmom_conc_diag_mult} with Theorem \ref{thm:score_m_success_guarantee}, we can prove Theorem 4.3:

\begin{proof}
    Define \begin{align*}
        \varepsilon_1 & := 4 c(\tau)\sqrt{\log\left(\frac{4}{(1-\tau)^2}\frac{1}{\delta}\right)\frac{\tr(\mathbf{\Sigma}_{\mathbf{\Gamma}_0})}{n}}, &
        \varepsilon_2 :=2c(\tau)\sqrt{\log\left(\frac{4}{(1-\tau)^2}\frac{1}{\delta}\right)\frac{\tr(\mathbf{\Sigma}_{g_0})}{n}}.
    \end{align*}
    Treating $\hat{\mathbf{\Gamma}}_K+\beta\diag(\hat{\mathbf{\Gamma}}_K)$ by Lemma~\ref{lemma:gmom_conc_diag_mult} and $\hat{g}_K$ by Theorem~ 3.1 (together with the inclusion of events $\{\|\cdot\|_\infty > \text{const}\}\subset\{\|\cdot\|_2> \text{const}\}$), applying a union bound yields that with probability at least $1 - 2\delta$ \begin{align*}
        \|\hat{\mathbf{\Gamma}}_K + \beta\diag(\hat{\mathbf{\Gamma}}_K)\|_\infty&\leq \varepsilon_1 / 2 < \varepsilon_1, & \|\hat{\mathbf{g}}_K - \mathbf{g}_0\|_\infty&\leq\varepsilon_2/2 < \varepsilon_2.
    \end{align*}
    Furthermore, the growth condition on $n$ ensures that
    $$d_{\bs{\theta}_0} \varepsilon_1 \leq \alpha /(6c_{\mathbf{\Gamma}_0})$$
    and, by construction,
    $$\lambda > 3(2 - \alpha)\max(c_{\bs{\theta}_0}\varepsilon_1, \varepsilon_2)/\alpha.$$
    The claim thus follows from Theorem~\ref{thm:score_m_success_guarantee}.
\end{proof}

\section{Additional simulations}
\subsection{Additional choices for the sample size n in Section 5.1}
Results of the experiment described in Section 5.1 of the main paper for $n=200$ and $n=5000$ are reported in Figure \ref{fig:sqr_roc_appendix}. The figure supports the conclusions from Section 5.1. The classic score matching procedure is almost non-informative in the experiment with $n=5000$ under contamination, highlighting the improvement the GMoM can have on robustness.

\begin{figure}
    \centering
    \includegraphics[width=\linewidth]{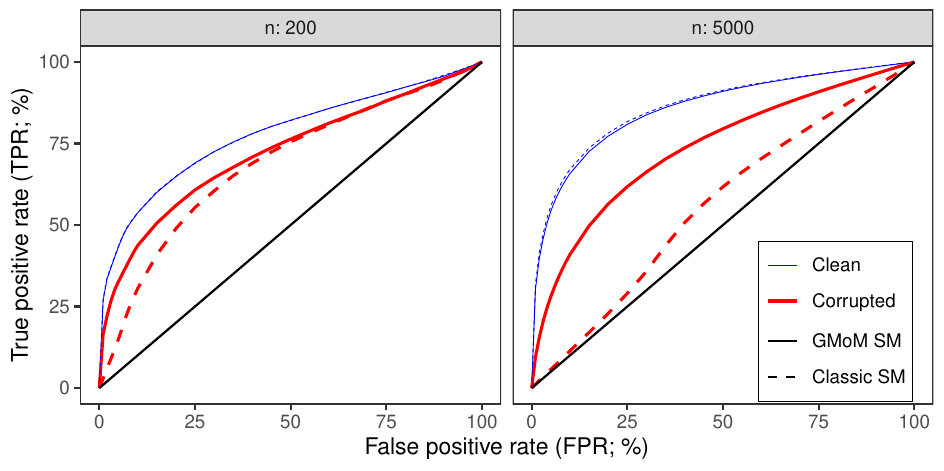}
    \caption{ROC curves for support recovery in the square root model. The experimental setup is the same as for Figure 1 of the main paper; the difference is the values for the sample size $n$.}
    \label{fig:sqr_roc_appendix}
\end{figure}

\subsection{Additional contamination scenarios}
Figure \ref{fig:additional_contamination} shows how the experiments in Section 5.1 of the main paper are affected by changes to the contamination scenario. Concretely, we consider the experiment with $n=1000$. 

For the left hand side of Figure \ref{fig:additional_contamination}, the contamination percentage was increased from $5\%$ to $10\%$, while the contamination distribution was maintained to be Pareto. As expected, when there is more contaminated samples, the GMoM version of score matching outperforms the classic version even more clearly.

For the right hand side of Figure \ref{fig:additional_contamination}, the contamination percentage was set to $5\%$ again, but the contaminating samples were drawn from a Gaussian graphical model. The dependence network of the contaminating distribution was chosen at random independently of the network underlying the uncontaminated sample. To make the contaminated samples not blatantly inconsistent with the true model, their absolute value was taken such that the support constraint of the square root model is satisfied. As the ROC curves show, the GMoM version outperforms the classic version significantly in this contamination setting, albeit the absolute difference is relatively small. In a way, it is surprising the GMoM has a significantly better ROC curve at all, given that it downweighs based on magnitude and not on semantics.

\begin{figure}
    \centering
    \begin{subfigure}{0.5\textwidth}
    \includegraphics[width=\linewidth]{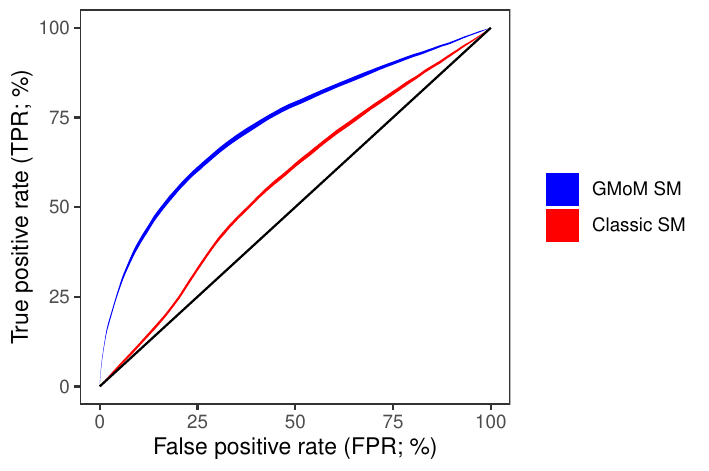} 
    \end{subfigure}\hfill
    \begin{subfigure}{0.5\textwidth}
    \includegraphics[width=\linewidth]{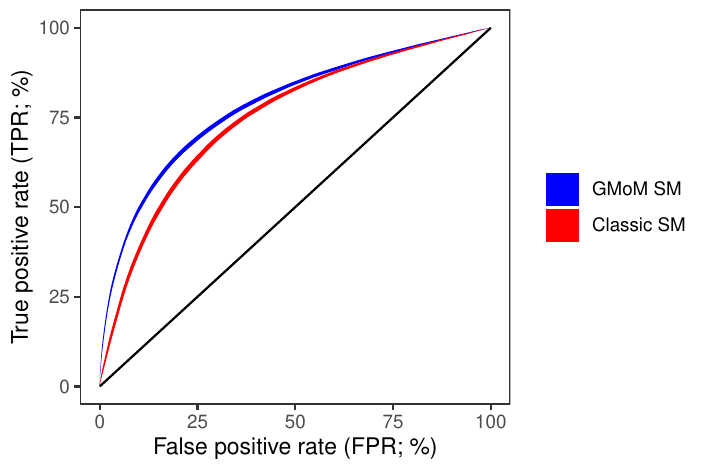}
    \end{subfigure}

    \caption{ROC curves for support recovery in the square root model under contamination only. $95\%$ confidence bands are shown. Left: $10\%$ Pareto contamination; Right: $5\%$ contamination with Gaussian data of different dependence structure.}
    \label{fig:additional_contamination}
\end{figure}

\subsection{Results for Gaussian graphical models}
In this section, we apply regularized score matching (Classic SM) and our extension using the GMoM (GMoM SM) to simulated data from the familiar class of Gaussian graphical models (GGMs). We compare their support recovery performance in terms of ROC to that of GLASSO \citep{glasso}, a widely adopted tool for estimating sparse GGMs.

Again, we consider a scenario where $\varepsilon=5\%$ of the observations are contaminated. Here, contaminated observations are replaced with draws from a Gaussian with iid components, having as variance $10$ times the maximum component variance of the uncontaminated model. To make for a fair comparison, we provide GLASSO with a robust covariance estimate in the contaminated case. Specifically, we use the MAD-Spearman combination theoretically treated for GLASSO in \citep{Loh2018}.

Data was generated from a $m=100$ dimensional Gaussian graphical model on an Erdős–Rényi graph with $100$ edges. In the spirit of the experiments in section 4.1 of \citep{LinHighDimScoreMatching}, we consider the borderline high-dimensional scenario $n=m=100$ samples. The number of blocks for the GMoM SM was set to $4\varepsilon n= 20$, thus again being conservative on uncorrupted data and being adapted to the contamination amount $\varepsilon$ on corrupted data. The diagonal multiplier was not needed since $n\geq m$. The penalty parameter $\lambda$ was varied to cover the entire ROC space. No dampening function $\mathbf{h}$ is needed for the Gaussian, as the domain is unrestricted.

ROC curves based on $100$ independent Monte Carlo simulations are displayed in Figure \ref{fig:sm_vs_glasso}. On uncorrupted data, all three methods have practically identical ROC curves, which is in line with the findings from \cite{LinHighDimScoreMatching}. On corrupted data, classic SM performs worse than the two robust methods, in line with the experiments from the main paper. The robust GLASSO and GMoM SM have very similar ROC curves, with GLASSO performing a bit better for high FPRs and GMoM SM a bit better for low FPRs. To conclude, the experiment shows that GMoM SM is a strong contender for estimating the support of sparse graphical models, especially when a part of the observations has been contaminated.

\begin{figure}
    \centering
    \includegraphics[width=0.9\linewidth]{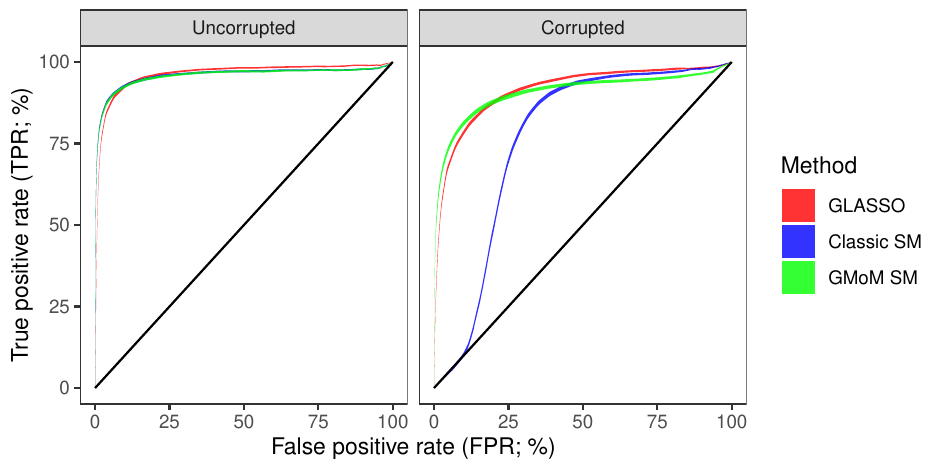}
    \caption{ROC curves for support recovery in the Gaussian graphical model. Right: $5\%$ of observations have been contaminated. Line width of ROC curves shows a $95\%$ confidence band.}
    \label{fig:sm_vs_glasso}
\end{figure}

\end{document}